\documentclass[twoside]{article}

\usepackage[accepted]{aistats2026}

\newif\ifsubmit
\submittrue  

\usepackage[round]{natbib}

\usepackage[utf8]{inputenc} 
\usepackage[T1]{fontenc}    
\usepackage{hyperref}       
\usepackage{url}            
\usepackage{booktabs}       
\usepackage{amsfonts}       
\usepackage{nicefrac}       
\usepackage{microtype}      
\usepackage[dvipsnames]{xcolor}
\usepackage{amsmath}
\usepackage{amssymb}
\usepackage{amsthm}
\usepackage{graphicx}
\usepackage{bbm}
\usepackage{array}
\usepackage{tcolorbox}
\usepackage{soul}
\usepackage{multirow}
\usepackage{enumitem}
\usepackage{adjustbox}
\usepackage{tabularx}
\usepackage{wrapfig}
\usepackage{listings}       

\newtheorem{assumption}{Assumption}
\newtheorem{theorem}{Theorem}
\newtheorem{remark}{Remark}
\newtheorem{lemma}{Lemma}

\newcolumntype{H}{>{\setbox0=\hbox\bgroup}c<{\egroup}@{}}

\newcommand{\ci}[3]{%
    \shortstack{#1 \\[-0.1em] {\scriptsize (#2,\,#3)}}%
}

\title{LLMs Judging LLMs: A Simplex Perspective}

\begin{document}

%
\runningtitle{LLMs Judging LLMs: A Simplex Perspective}

%

\twocolumn[

\aistatstitle{LLMs Judging LLMs: A Simplex Perspective}
\runningauthor{Vossler, Xia, Mai, Subbaswamy, Feng}
\aistatsauthor{Patrick Vossler$^{1*}$ \And Fan Xia$^{1*}$ \And Yifan Mai$^{2}$ \And Adarsh Subbaswamy$^{3}$ \And Jean Feng$^{1}$}

\vspace{0.5em}

\aistatsaddress{$^1$University of California, San Francisco \And $^2$Stanford University \And $^3$University of Maryland, Baltimore} ]

\renewcommand{\thefootnote}{\fnsymbol{footnote}}
\footnotetext[1]{Equal contribution}
\renewcommand{\thefootnote}{\arabic{footnote}}

\begin{abstract}
Given the challenge of automatically evaluating free‐form outputs from large language models (LLMs), a common solution is to use LLMs themselves as judges, without any gold-standard scores.
Implicitly, this practice accounts for only sampling variability (\textit{aleatoric uncertainty}) and ignores uncertainty about judge quality (\textit{epistemic uncertainty}).
While this is justified if judges are perfectly accurate, it is unclear when such an approach is theoretically valid and practically robust.
We study these questions for the task of ranking LLM candidates from a novel geometric perspective: for $M$-level scoring systems, both LLM judges and candidates can be represented as points on an $(M-1)$-dimensional probability simplex, where geometric concepts (e.g., triangle areas) correspond to key ranking concepts.
This perspective yields intuitive theoretical conditions and visual proofs for when rankings are identifiable; for instance, we provide a formal basis for the ``folk wisdom'' that LLM judges are more effective for two-level scoring ($M=2$) than multi-level scoring ($M>2$).
Using this geometric intuition, we design Bayesian priors that encode epistemic uncertainty and vary the priors to conduct sensitivity analyses.
Experiments on LLM benchmarks show that rankings based solely on LLM judges are robust in many but not all datasets, underscoring both their widespread success and the need for caution.
Our Bayesian method achieves substantially higher coverage rates than existing procedures by modeling epistemic uncertainty.
\end{abstract}

\section{INTRODUCTION}

Scalable benchmarking of large language models (LLMs) is increasingly critical given the rapid proliferation of models, model updates, and new benchmark datasets.
While multiple choice or numerical answers can be verified algorithmically, many benchmark tasks allow free-form text responses that are more difficult to verify, such as clinical reasoning expressed as natural language or mathematical proofs involving multi-part LaTeX formulas.
While the current gold standard of consensus voting by multiple human experts is effective, it is often prohibitively expensive and difficult to scale.

In response to these challenges, recent work has proposed using LLMs themselves as judges (Figure~\ref{fig:task}) \citep{Zheng2023-wp}.
Here we refer to evaluator LLMs as ``judges'' and evaluated LLMs as ``candidates.''
This now-common practice effectively treats LLM judges as perfectly accurate, ignoring any uncertainty about their quality and acknowledging only sampling variability (e.g., through bootstrapping or confidence intervals).
Borrowing terms from epistemology, the first source of uncertainty is systematic or \textit{epistemic}; due to the lack of gold-standard data, one is uncertain about and cannot determine the true quality of the judges.
The second source of uncertainty is stochastic or \textit{aleatoric}; due to sampling variability, there is noise in the returned judge scores but this type of uncertainty can be reduced by increasing sample size.
\textit{The goal of this work is to understand how epistemic uncertainty impacts our ability to rank LLM candidates.}
We do so by answering two questions: (i) from a theoretical standpoint, what are sufficient conditions for an LLM-as-a-judge pipeline to identify the true rankings
and (ii) from a methodological standpoint, how can we quantify our confidence in estimated rankings when we do not know if these theoretical conditions hold in practice?

\begin{figure}
    \centering
    \includegraphics[width=\linewidth]{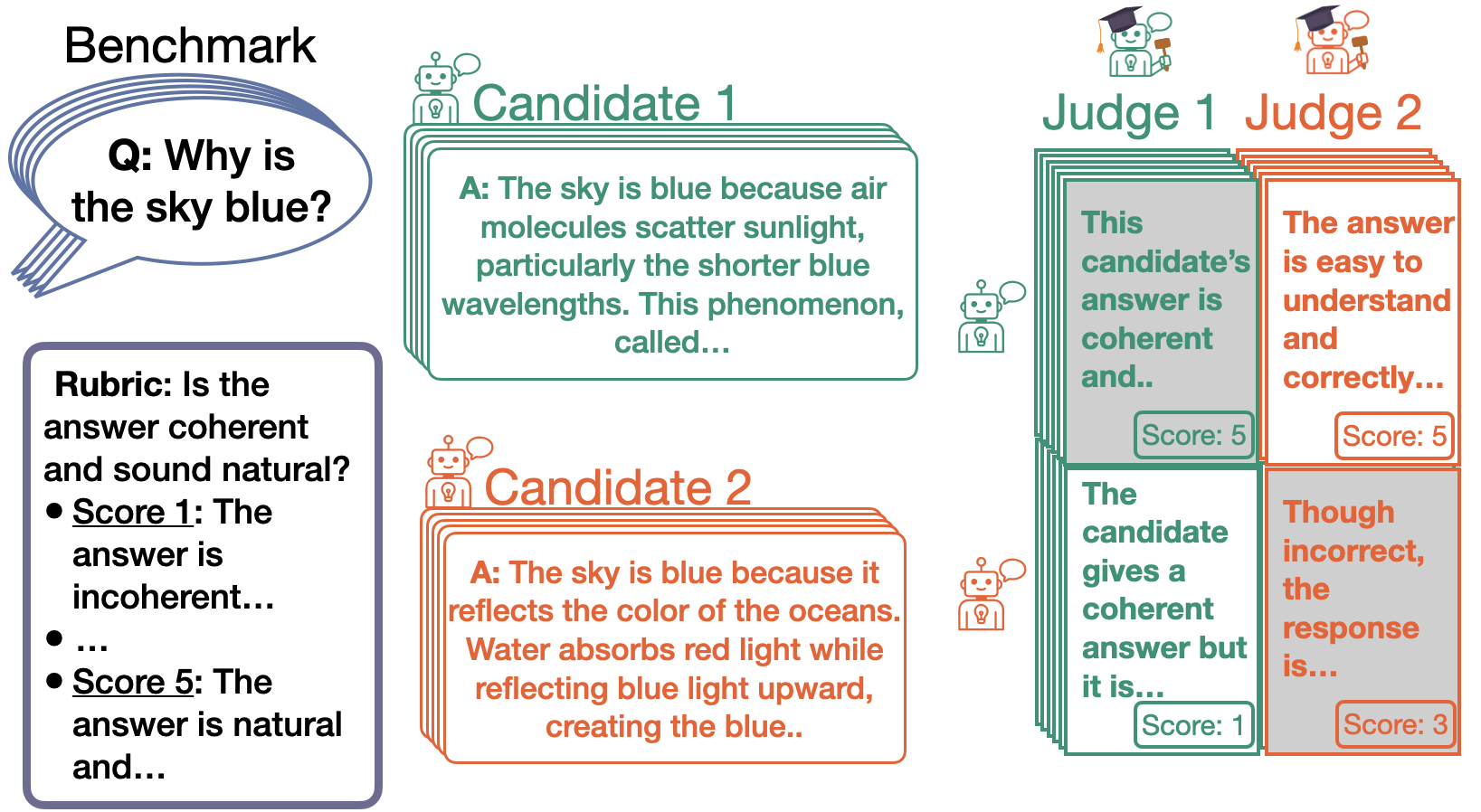}
    \caption{LLM judge workflow: For each benchmark question, LLM judges score each candidate's answer according to a rubric.
    Candidates are ranked based on their judge-assigned scores.
    Shaded boxes indicate cases where the same LLM serves as both candidate and judge (self-judging).
    }
    \ifsubmit
    \fi
    \label{fig:task}
\end{figure}

We show that a geometric perspective clarifies the role of epistemic uncertainty and guides both theoretical and methodological development.
The key idea is that we can simultaneously visualize judge quality and candidate score distributions on a probability simplex, where geometric concepts (e.g., triangle areas) correspond to key ranking concepts.
Epistemic uncertainty manifests as uncertainty in the location of judge points, while aleatoric uncertainty manifests as stochasticity of candidate points.
This representation allows us to reason about how the relative positions of these points affect identifiability and estimation of the true rankings.

Using this framework, we find that constancy of LLM judge behavior is sufficient for ranking identifiability in two-level scoring systems.
However, this sufficiency breaks down when ranking across three or more levels, uncovering a surprising phase transition.
These theoretical results motivate the development of a Bayesian framework for estimating ranks, where we construct minimally-informative geometric priors that encode our prior beliefs about judge quality.
We then conduct sensitivity analyses of estimated rankings by varying the strength of the Bayesian priors, which corresponds to varying the level of epistemic uncertainty.

Applying our framework to five LLM benchmark datasets, our sensitivity analyses reveal that estimated candidate rankings are robust in some cases (e.g., GPQA and SummEval) but not others (e.g., Omni-MATH).
These results explain the practical success of LLM judges, but also emphasize the importance of their careful, measured use.
The Bayesian framework also achieves substantially higher coverage of the true rankings than existing methods, since it captures both aleatoric and epistemic uncertainty whereas standard approaches capture only the former.
These results underscore how incorporating epistemic uncertainty is necessary for reliable rankings from LLM judges.
\ifsubmit
Code for reproducing this work and running the Bayesian inference pipeline is available at \url{https://github.com/jjfenglab/judging-llms-on-a-simplex}.
\fi

\section{RELATED WORK}

\textbf{LLM-as-a-judge and their implicit assumptions.}
While LLM judges have emerged as a scalable alternative to traditional metrics that require reference outputs \citep{Papineni2001-dx, Lin2004-bz, Wang2023-hy,Chiang2023-we, Li2024-ee, Gu2024-im}, current approaches may not fully capture all sources of uncertainty in their rankings.
LLM judges exhibit systematic biases such as position bias and verbosity bias~\citep{Zheng2023-wp, Koo2024-gd, Wei2025-mp}, echoing known phenomena in human judgment---the halo effect~\citep{Thorndike1920-halo} and anchoring~\citep{Tversky1974-anchor}.
Proposed mitigations---score averaging, juries, judge debates, rubrics \citep{Liang2023-ow, Verga2024-vf, Kalra_undated-rn, Chan2024-dj,Databricks2024-rb, Lee2024-qb}---implicitly assume that judges are perfect or that their errors cancel out in expectation, ignoring uncertainty about judge quality.
While \citet{Guerdan2025-to} discussed the limitations of LLM judges when there is no true agreed-upon rating scale, we prove that even with agreed-upon scales, fundamental identifiability limits exist that no averaging or ensemble method can overcome.

\textbf{Uncertainty quantification without ground truth.} Existing uncertainty quantification methods cannot address the fundamental challenge of epistemic uncertainty about judge quality.
Bootstrapping \citep{Goldstein1996-mj, Xie2009-wv} only quantifies sampling variation, leading to undercoverage when assumptions fail.
Methods like prediction-powered inference \citep{Chatzi2024-cy, Angelopoulos2023-tv}, conformal inference \citep{Jung2025-lt}, or consensus models \citep{Raykar2010-dd,Welinder2010-rc} can integrate LLM judging if some subset of gold-standard labels are available for \textit{each} new model and benchmark, limiting scalability.
Bradley-Terry models \citep{Bradley1952-kj, Rao1967-fv, Herbrich2006-jx, Ameli2025-fg}, widely used in rating systems like Elo and Chatbot Arena, assume strong stochastic transitivity, restricting their  applicability to settings where pairwise comparisons uniquely determine a global ranking.

\textbf{Crowdsourcing.}
A key statistical problem in crowdsourcing is to recover the confusion matrix and impute the true labels for observations given labels from noisy annotators \citep{Dawid1979-vl}.
This is related to the problem of judging candidates given LLM judges, but a key difference is this work's focus on \textit{ranking} candidates.
The ranking perspective allows us to use assumptions that are weaker than those typically considered in the crowdsourcing literature \citep{Welinder2010-rc, Ibrahim2019-dj}.

\textbf{Imperfect reference standards.} Using LLMs as imperfect judges parallels evaluation of medical diagnostics with imperfect reference standards \citep{Reitsma2009-rn, Umemneku_Chikere2019-kt, Sun2025-si, Sun2024-sj}.
Our results expand on ideas used in this field, beyond the classical setting of evaluating binary diagnostic tests to the case of multi-level ratings, which are commonly used for LLM evaluation.
The geometric arguments substantially extend \citep{Fienberg1970-dt, Black2002-wq, Jones2010-qa, Duan2020-lg}, which focused on identifying diagnostic test performance rather than ranking and only considered the case with binary diagnostic tests.

\section{A GEOMETRIC PERSPECTIVE}
\label{sec:geometric-framework}

We study the LLM judge pipeline shown in Figure~\ref{fig:task}, where $J$ LLM judges score answers by $K$ candidate LLMs to questions from a benchmark task.
Formally, denote the spaces of questions and answers as $\mathcal{Q}$ and $\mathcal{A}$, respectively.
A benchmark task is defined by a distribution over questions $Q \sim P_{\mathcal{Q}}$ and a true scoring function $s^*: \mathcal{Q} \times \mathcal{A} \to \{1,\ldots,M\}$.
Each candidate $k$ is a (potentially stochastic) function $f_k: \mathcal{Q} \to \mathcal{A}$ that generates answers, with true score $S^*_k = s^*(Q, f_k(Q))$ for question $Q$.
Candidate $k$'s true score distribution is denoted $\pi_k = (\pi_{k,1}, \ldots, \pi_{k,M})$, where $\pi_{k,m} = \Pr(S^*_k = m)$ is the frequency with which the candidate achieves true score $m$.
The goal is to recover the true ranking of candidates with respect to their expected true scores $\mathbb{E}[S^*_k] = \sum_{m=1}^M m \cdot \pi_{k,m}$.

We study the setting where both true scoring function $s^*$ and the true scores $S^*_k$ are unobserved.
Instead, we rely on the LLM judge pipeline to obtain judge-assigned scores, which can be viewed as noisy proxies.
Each judge $j=1,\cdots,J$ is represented as a function $\hat{s}_j: \mathcal{Q} \times \mathcal{A} \to \{1,\ldots,M\}$.
So given question $Q$ and candidate $k$'s answer, the judge assigns score $\hat{S}_{k}^{(j)} = \hat{s}_j(Q, f_k(Q))$.
Given infinite questions drawn from the benchmark task, the distribution of judge-assigned scores for candidate $k$ is $\gamma^{(j)}_k = (\gamma^{(j)}_{k,1}, \ldots, \gamma^{(j)}_{k,M})$, where $\gamma^{(j)}_{k,m} = \Pr(\hat{S}_{k}^{(j)} = m)$ is the frequency that judge $j$ assigns score $m$ to the candidate's answers.

\begin{figure}
    \centering
    \includegraphics[width=\linewidth]{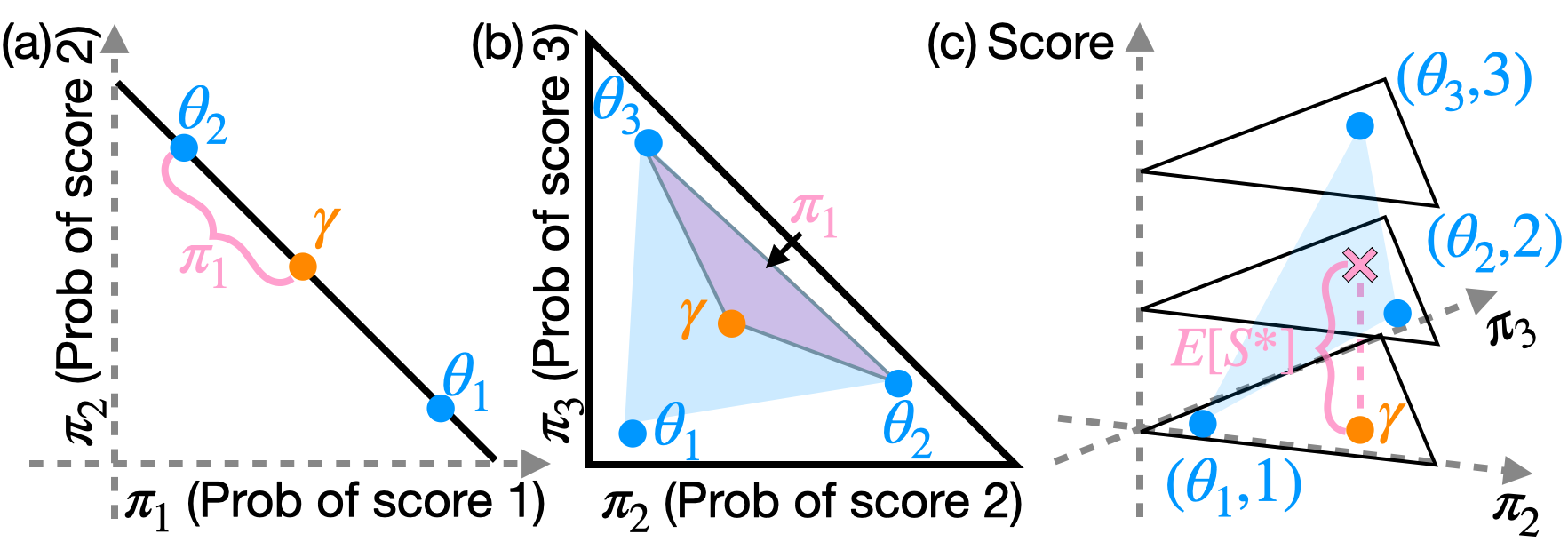}
    \caption{
    The true prevalence of the true scores corresponds to barycentric coordinates; (a) and (b) highlight prevalence $\pi_1$ in 2- and 3-level scoring systems, respectively. A candidate's expected score corresponds to the height of vertical projection in the augmented space, as illustrated in (c).
    }
    \ifsubmit
    \vspace{-0.3cm}
    \fi
    \label{fig:prelim_simplex}
\end{figure}

Using simple geometric arguments, we can relate the true score distribution $\pi_k$ to the observed score distribution $\gamma_k^{(j)}$ using judge $j$'s confusion matrix for candidate $k$.
In particular, let the $m$th column of the judge's confusion matrix for candidate $k$ represent the distribution of assigned scores for answers with true score $m$, i.e.,
\small
\begin{align*}
\theta^{(j)}_{m,k} = \big(\Pr(\hat{S}_{k}^{(j)} = 1 | S^*_k = m), \ldots, \Pr(\hat{S}_{k}^{(j)} = M | S^*_k = m)\big).
\end{align*}
\normalsize
Each of the judge's confusion matrix columns $\theta^{(j)}_{m,k}$ for $m=1,\cdots, M$ lie on the $(M-1)$-dimensional probability simplex and, critically, the candidate's judge-assigned score distribution $\gamma_k^{(j)}$ is the convex combination of these columns with respect to true score distribution, i.e.,
$\gamma^{(j)}_k = \sum_{m=1}^{M} \pi_{k,m} \theta^{(j)}_{m,k}.$
As examples, Figure~\ref{fig:prelim_simplex}a  shows how the judge and candidate in a 2-level scoring system all fall on a line segment and Figure~\ref{fig:prelim_simplex}b shows how these points in a 3-level scoring system all fall within a triangle.
Using this geometric perspective, we can then establish equivalences between geometric concepts to key ranking concepts:

\textbf{1. True score distributions correspond to barycentric coordinates.}
$\pi_{k}$ corresponds precisely to the \textit{barycentric} coordinates for the weighted centroid $\gamma^{(j)}_k$ relative to judge vertices $\theta^{(j)}_{m,k}$, which are also known as \textit{areal} coordinates because they can be expressed as ratios of simplex sub-areas.
For instance, for $M=3$, $\pi_{k,1}$ equals the area of the subtriangle from the weighted centroid to the 2nd and 3rd vertices divided by the area of the triangle defined by all vertices (Figure~\ref{fig:prelim_simplex}b), i.e.,
\begin{align}
\pi_{k,1} = \frac{
\text{Area of } \triangle_{(\gamma_k, \theta_2, \theta_3)}
}{
\text{Area of } \triangle_{(\theta_1, \theta_2, \theta_3)}
},
\label{eq:barycentric}
\end{align}
with $\pi_{k,2}$ and $\pi_{k,3}$ defined analogously.
This geometric characterization extends to any $M$.

\textbf{2. Expected scores map to height in augmented space}.
If we augment the probability simplex with an additional axis representing score values and lift each judge vertex $\theta^{(j)}_{m,k}$ to height $m$, the vertices form an augmented simplex (Figure~\ref{fig:prelim_simplex}c).
Then, as proved in the Appendix, candidate $k$'s expected score $\mathbb{E}[S^*_k]$ corresponds precisely to the height of its vertical projection onto this augmented simplex.

\textbf{3. Epistemic uncertainty manifests as unknown positions of judge vertices.}
When multiple configurations of judge vertices are compatible with the observed set of candidate points, the location of the judge vertices cannot be determined.
Thus epistemic uncertainty of judge quality is represented by uncertainty in the location of judge vertices.

\textbf{4. Aleatoric uncertainty manifests as stochasticity of candidate points.}
Aleatoric uncertainty is sampling variability, which corresponds to us observing the empirical judge-assigned score distributions $\hat \gamma_{k,n}^{(j)}$ for a benchmark dataset with $n$ questions, rather than the expected judge-assigned score distribution $\gamma_{k}^{(j)}$.

\section{THEORETICAL LIMITS OF RANKING IDENTIFIABILITY}
\label{sec:theory}

\begin{figure}
    \centering
    \includegraphics[width=\linewidth]{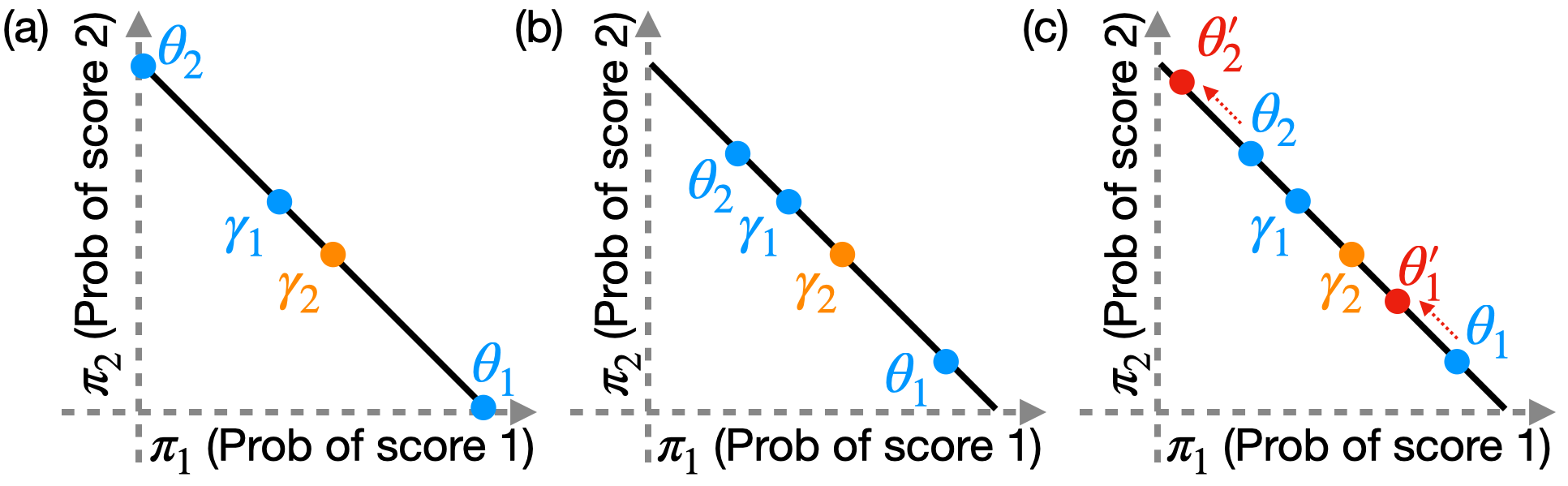}
    \caption{Visualization of judge assumptions for the 2-level scoring setting. We suppose there are two LLMs (1=blue, 2=orange), where both are candidates (with score distributions $\gamma_{1}$ and $\gamma_{2}$) and LLM 1 is a judge.
    (a) Perfect Judge assumes judge vertices $(\theta_{1}, \theta_{2})$ are at the extremes of the 1-dimensional probability simplex.
    (b) Strong Constancy assumes the vertex positions for the judge are same across all candidates.
    (c) Moderate Constancy assumes the vertex positions for the judge only differ when self-judging (indicated by the red shifted vertices).}
    \ifsubmit
    \vspace{-0.4cm}
    \fi
    \label{fig:identifiability}
\end{figure}

Using this geometric framework, we now investigate when true rankings can or cannot be recovered from LLM judge scores without ground truth labels.
To understand the theoretical limits, we assume access to infinite data samples, resulting in zero aleatoric uncertainty and leaving only epistemic uncertainty.
To clarify these limits, we consider assumptions that idealize judge behavior to varying degrees; Section~\ref{sec:bayesian-ranking} relaxes these idealizations for practical inference.

In the most general case with zero constraints on the LLM judge behaviors, our geometric visualization will have $J \times K\times M $ points representing the performance of the $J$ judges across the $K$ candidates as well as $J \times K$ points representing the score distributions assigned by each judge to each candidate.
The true rankings are generally nonidentifiable in this setting, so we must impose structure on judge behavior.
Thus we consider assumptions, that when imposed, may sufficiently reduce epistemic uncertainty to make the true rankings identifiable.

The strongest assumption is to assume the judges are perfect.
The confusion matrices collapse to the identity matrix, placing all judge vertices at the simplex corners (Figure~\ref{fig:identifiability}a).
If true, this assumption would obviously make the true rankings identifiable for any value of $J$, $M$, and/or $K$.
However, this is likely unrealistic.

A weaker and somewhat more realistic assumption is that the judge's performance is perfectly consistent across candidates (\textit{Strong Constancy}).
Mathematically, this corresponds to the confusion matrix being the same across candidates, i.e., the vertices for each judge $j$ are shared across candidates (Figure~\ref{fig:identifiability}b).
\begin{assumption}[Strong Constancy]
The confusion matrix of judge $j$ is identical across all $K$ candidates: For each $m$, there is some $\theta^{(j)}_m$ such that
$\theta^{(j)}_{m,k} = \theta^{(j)}_m$ for $k=1,\cdots,K$.
\label{assumption:constancy}
\end{assumption}

A weaker assumption that allows for self-preference bias \citep{Koo2024-gd} is to suppose that the LLM judge is perfectly consistent as long as it is not judging itself (\textit{Moderate Constancy}).
This can be visualized by a separate set of vertices when the LLM is judging itself but a shared set of vertices otherwise (Figure~\ref{fig:identifiability}c).
\begin{assumption}[Moderate Constancy]
Judge $\hat{s}_j$ has an identical confusion matrix for all non-self candidates: For each $m$, there is some $\theta^{(j)}_m$ such that $\theta^{(j)}_{m,k} = \theta^{(j)}_m$ for all $k\ne j$.\footnote{For simplicity, the theoretical results study the case where self-judging means the same LLM judges itself. In practice, we exclude judges from judging their model family.}
\label{assumption:nonself_constancy}
\end{assumption}


While one may consider even weaker assumptions, we study the implications of these two constancy conditions on ranking identifiability, as even here the results are nuanced.
These theoretical results then motivate our framework in the next section for estimating ranks and quantifying uncertainty.



\subsection{2-level scoring systems}
\label{sec:2-level-theory}

Let us consider the simplest setting: a 2-level scoring system with a single judge that satisfies \textit{Strong Constancy}.
In this setting, all candidates must lie on the line segment between the two judge vertices $\theta_1$ and $\theta_2$.
Figure~\ref{fig:prelim_simplex} (left) provides a visual proof why the true ranking is identifiable: for 2-level scoring, ranking candidates by expected true score reduces to ranking them by their prevalence of the higher score, which is determined by their order along the line segment.
This is true even though we do not know the exact positions of the judge vertices and even if the LLM judges are far from perfect.
In fact, our only requirement is that the judges are not adversarial (to avoid the so-called ``label-flipping problem'' \citep{Sun2025-si}), i.e.,
\begin{assumption}
Judge $j$'s probability of assigning the lowest score when the true score is equal to $m$ decreases with respect to $m$.
\label{assumption:no_stupid}
\end{assumption}
For 2-level scoring, this is a weak condition as it only requires the judges to perform better than random.

The result under \textit{Moderate Constancy} is analogous.
Using similar geometric reasoning (see Appendix), we find that the only difference is that at least two judges and four candidates are needed to obtain enough information to recover the true ranking.

\begin{tcolorbox}[boxsep=0pt, left=3pt, right=3pt, top=2pt]
\begin{theorem}
For 2-level scoring, any of the following conditions are sufficient for candidate rankings to be identifiable from the distribution of judge-assigned scores:
\begin{itemize}[itemsep=0pt, parsep=0pt]
    \item[(i)] There is $J=1$ judge and it satisfies Assumptions~\ref{assumption:constancy} and \ref{assumption:no_stupid}.
    \item[(ii)] There are $K\ge 4$ candidates, $J\ge 2$ judges, and the judges satisfy Assumptions~\ref{assumption:nonself_constancy} and \ref{assumption:no_stupid}.
\end{itemize}
\label{thrm:strong_constancy2}
\end{theorem}
\end{tcolorbox}

\textit{Remark}. For completeness, we present in the Appendix necessary and sufficient conditions for recovering rankings in two-level systems. The geometry is complex, so we defer the details to the appendix.

\subsection{3+-level scoring systems}
\label{sec:3-level-theory}

Results for scoring systems with 3 or more levels are qualitatively different.
Below, we present geometric arguments for 3-level scoring systems, but the same arguments extend to more levels.
To build intuition, we first discuss the tasks of (i) ranking the prevalence of each score and then (ii) ranking the expected scores.

\textbf{(i) Ranking prevalences.}
Without loss of generality, suppose we wanted to rank the prevalence of true score $m=1$.
As mentioned in Section~\ref{sec:geometric-framework}, this is equivalent to ranking the barycentric coordinates $\pi_{k,1}$ across candidates $k=1,\cdots,K$, which are defined as the area ratios \eqref{eq:barycentric}.
Under strong constancy, the denominators of these ratios are the same for all $k$, so ranking $\pi_{k,1}$ reduces to ranking the areas of the subtriangles $\triangle_{(\gamma_k, \theta_2, \theta_3)}$
Nevertheless, we find that the rankings are non-identifiable in general, even under strong constancy.
Figure~\ref{fig:three_level_proof} provides a visual proof:
panels (a) and (b) show two equally valid judge configurations ($\{\theta_1, \theta_2, \theta_3\}$ versus $\{\theta'_1, \theta'_2, \theta'_3\}$) that can explain candidate score distributions $\gamma_{k=1}$ and $\gamma_{k=2}$, as both judge configurations envelop the candidate points.
However, the areas flip between configurations.
In panel (a), subtriangle $\triangle_{(\gamma_1, \theta_2, \theta_3)}$ has greater area than $\triangle_{(\gamma_2, \theta_2, \theta_3)}$, while in panel (b) this relationship reverses.
Thus the two panels show that even under strong constancy, the rankings of the score prevalences cannot be recovered if there is no prior knowledge about judge quality.
\begin{figure}
    \centering
    \includegraphics[width=0.8\linewidth]{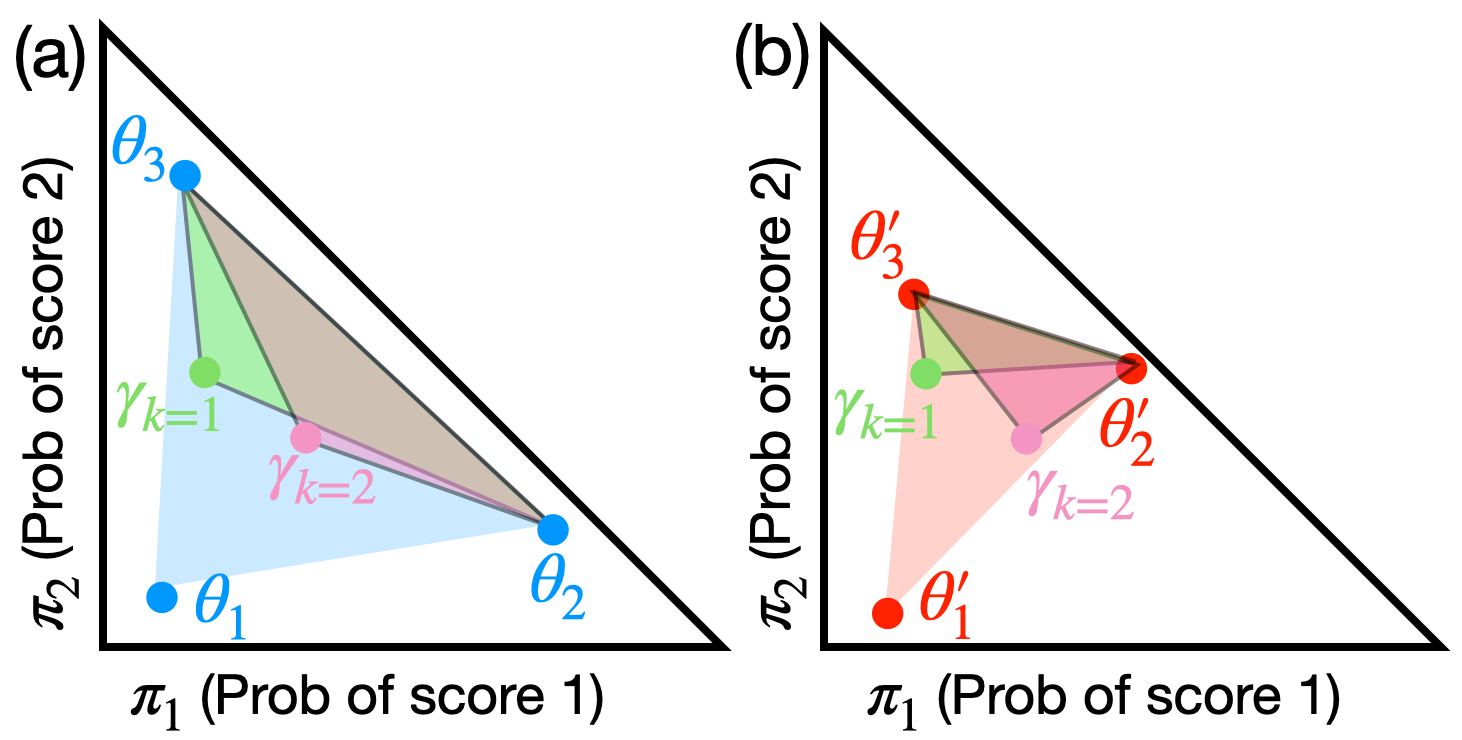}

    \ifsubmit
    \vspace{-0.4cm}
    \fi

    \caption{Non-identifiability in 3-level scoring. Same candidate positions (green/pink) explained by different judge configurations yield opposite prevalence rankings: $\pi_{1,1} > \pi_{2,1}$ in (a) vs. $\pi_{1,1} < \pi_{2,1}$ in (b).}

    \ifsubmit
    \vspace{-0.3cm}
    \fi

    \label{fig:three_level_proof}
\end{figure}

\textbf{(ii) Ranking expected scores}.
Given the nonidentifiability of prevalence rankings, the ranking of expected true scores is also nonidentifiable without gold standard labels.
To see this, recall that each candidate's expected score corresponds to the height of its vertical projection onto the augmented simplex (Figure~\ref{fig:prelim_simplex}c); these heights change when the judge vertices shift, which can lead to a reordering of candidates.
This argument applies under either constancy assumption and for any number of judges.

Together, these results illustrate that epistemic uncertainty can significantly affect recoverability of the true rankings, even when aleatoric uncertainty is removed.
\begin{tcolorbox}[boxsep=0pt, left=3pt, right=3pt, top=2pt]
\begin{theorem}
Consider the 3+-level scoring setting with all judges either satisfying Assumptions~\ref{assumption:constancy} and \ref{assumption:no_stupid} or Assumptions~\ref{assumption:nonself_constancy} and \ref{assumption:no_stupid}.
Given only the distribution of judge-assigned scores, there exist candidates whose prevalence and/or expected-score rankings cannot be identified.
\label{thrm:strong_constancy3}
\end{theorem}
\end{tcolorbox}
\ifsubmit
\vspace{-0.2cm}
\fi

\subsection{Recoverability is dataset-specific}
The prior theoretical results show that recoverability depends not only on judge quality assumptions but also the problem setting.
In particular, Theorems~\ref{thrm:strong_constancy2} and \ref{thrm:strong_constancy3} uncover a dependence on the number of scoring levels, supporting the ``folk wisdom'' that LLM judges tend to be more effective for 2-level scoring systems and less effective for 3+levels \citep{Shankar2024-an, Husain2025-js}.
The geometric perspective further reveals that recoverability is, in fact, dataset-specific, as illustrated in the following two examples.
This means we need a way to rigorously quantify how sensitive the estimated rankings for a given dataset are to varying levels of epistemic uncertainty, which we address using Bayesian inference in the next section.


\textbf{Example 1: Dataset requiring stronger priors.}
Consider again the Figure~\ref{fig:three_level_proof} example and the task of ranking prevalence of score $m=1$.
Intuitively, the judge configuration in panel (a) may seem more probable than panel (b) if we believe that LLM judges can distinguish scores 2 and 3 moderately well.
To formalize this prior belief, we distill it using geometric arguments.
In particular, ranking $\pi_{k,1}$ is equivalent to ranking areas of subtriangles $\triangle{\gamma_k, \theta_2, \theta_3}$.
Because all the subtriangles have the same line segment $\overline{\theta_2 \theta_3}$ as the base, ranking areas is equivalent to ranking distances from each candidate $\gamma_k$ to the line $\overrightarrow{\theta_{2} \theta_{3}}$.
Because ranking distances only require knowing the \textit{slope} of this line, it is sufficient to state our prior belief about judge quality in terms of \textit{slopes}.
This is something we generally have a good prior about.
In contrast, it is generally more difficult to specify one's prior about the exact location of the judge vertices, since LLMs are often miscalibrated~\citep{Wang2024-gp}.

\textbf{Example 2: Dataset allowing weaker priors.}
Consider a candidate receiving mostly 1's versus another receiving mostly 3's from an LLM judge (see visualization in Appendix).
When score distributions are so different, the ranking between two candidates may be quite evident, even if we have limited prior knowledge on the performance of the LLM judges and even if the constancy assumptions do not hold.
However, we need a method to quantify when constancy violations flip rankings, as discussed in the next section.

\section{BAYESIAN FRAMEWORK}
\label{sec:bayesian-ranking}
The theoretical results show that epistemic uncertainty impacts our ability to rank candidates but that the level of impact is dataset-specific.
To model this, we design a Bayesian framework that encodes epistemic uncertainty through two priors: (i) how strongly the constancy assumption holds,  represented by random effects (Section~\ref{sec:relax}), and (ii) prior beliefs about judge discriminative ability, represented by slopes between judge vertices (Section~\ref{sec:judge_quality}).
We can then conduct sensitivity analyses that assess the ranking robustness to varying levels of epistemic uncertainty, by systematically varying the strength of the Bayesian priors.
Moreover, by defining appropriate hyperpriors, we can fully marginalize over our epistemic uncertainty when conducting posterior inference for candidate rankings.
We describe the key components of the Bayesian model here; the full specification appears in the Appendix.

\subsection{A base probability model}
As the base probability model, the assigned score $\hat{S}_{ik}^{(j)}$ by judge $j$ to candidate $k$'s answer to the $i$-th question,  given its true score $S_{ik}^*$, is modeled as independent draws from a multinomial distribution with parameter $\theta_{S_{ik}^*,k}^{(j)}$.
Marginalizing out the true latent scores, the likelihood of the observed data is

\ifsubmit
\vspace{-0.2cm}
\fi

\begin{align*}
    \prod_{i=1}^{n}
    \prod_{j=1}^J
    \prod_{k \ne j}
    \Big[\sum_{m=1}^M
    \underbrace{\Pr\left(\hat{S}^{(j)}_{ik}|S_{ik}^* = m; \theta_{m,k}^{(j)}\right)}_{=\theta_{m,k,\hat{S}^{(j)}_{ik}}^{(j)}}
    \underbrace{\Pr(S_{ik}^* = m)}_{=\pi_{k,m}}
    \Big],
\end{align*}
where $n$ is the number of questions.
To remove the influence of self-preference, note that we filter for $k\ne j$.

The base model assumes conditional independence given the true score for computational tractability.
In practice, residual correlation is still likely, as LLMs are trained on similar datasets.
This will be naturally addressed by random effects introduced in the next section, which is motivated from the different angle of relaxing the constancy assumption.

\subsection{Prior for relaxing constancy}
\label{sec:relax}
To relax the constancy assumptions (Assumptions~\ref{assumption:constancy} and \ref{assumption:nonself_constancy}), we model the prevalence of each judge-assigned score for a given candidate using random effects relative to a ``base'' score prevalence $\vec\pi_k$:
\begin{align*}
\vec\pi_{k}^{(j)} = (1 - W_{k} R_{j}) \vec\pi_k + W_{k} R_{j} Z_k
\end{align*}
where $Z_k \sim \text{Dirichlet}(\delta)$ determines the direction of the deviation, $R_j \sim \text{Beta}(\omega J, J)$ controls the judge-specific magnitude, and $W_k \sim \text{Beta}(\omega K, K)$ controls the candidate-specific magnitude.
The hyperparameter $\omega \in [0,\infty)$ controls the degree of constancy violation through the random effect magnitude, where $\omega=0$ enforces perfect constancy and increasing $\omega$ allows progressively larger violations.
For instance, $\omega=1$ expects random effects to contribute half of the effect, while $\omega = 8$ expects a contribution of 90\%.

\subsection{Prior for varying judge quality}
\label{sec:judge_quality}
\begin{figure}
    \centering
    \begin{minipage}[t]{0.4\linewidth}

      \ifsubmit
      \vspace{0pt}
      \fi

      \centering
      \includegraphics[width=\linewidth]{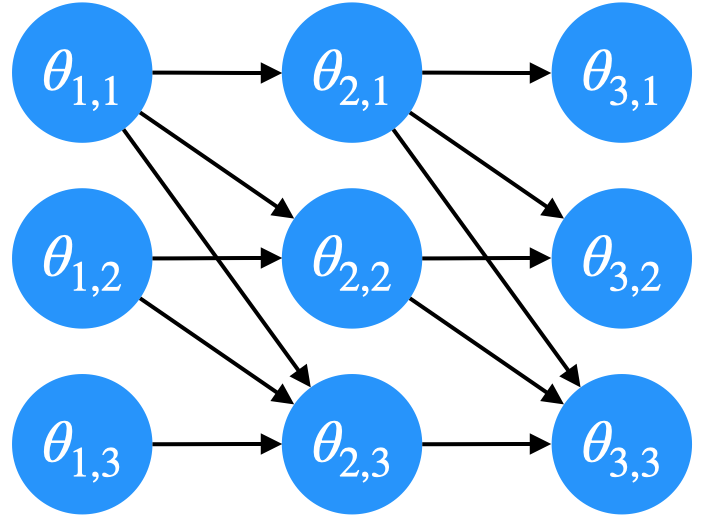}
    \end{minipage}
  \caption{
  Weight propagation framework for encoding judge quality priors in 3-level scoring, where outgoing edge weights must sum to 1. The probability of the judge assigning score $m_2$ when the true score is $m_1$, $\theta_{m_1, m_2}$, equals the weighted sum of all its parent nodes per the incoming edge weights.
  }

  \ifsubmit
  \vspace{-0.4cm}
  \fi

  \label{fig:transition}
  \end{figure}

To construct a Bayesian prior over slopes between judge vertices, we parameterize the judge vertices in terms of a weight propagation graph (Figure~\ref{fig:transition}), where each element in the confusion matrix is an average of its parent nodes weighted by edge weights $\alpha$, i.e., $\theta_{\texttt{v}}^{(j)} = \sum_{\texttt{u} \rightarrow \texttt{v}} \theta_{\texttt{u}}^{(j)} \alpha_{\texttt{u} \rightarrow \texttt{v}}$.
As such, the slopes between vertices are directly controlled in terms of edge weights.
This parameterization encodes the prior belief that judges tend to assign scores close to the true score more often than distant ones, while remaining flexible enough to accommodate varying degrees of judge quality.
The graph structure also enforces the monotonicity requirement (Assumption~\ref{assumption:no_stupid}) by construction.
We discuss alternative prior specifications considered in Appendix~E.4.

For each node $\texttt{v}$, we place a Dirichlet prior over the outgoing edge weights $\alpha_{\texttt{v}}$ indexed by a single hyperparameter $\beta_{\max}$, where increasing $\beta_{\max}$ corresponds to more discriminative judges.
For instance, in 3-level scoring, the Dirichlet prior for outgoing edges from node $\theta_{m_1,m_2}$ has parameters $\vec{\beta}_{m_1,m_2} = [1, 1 + \rho \cdot \beta_{\max}, 1]$ with $\rho \sim \text{Beta}(1,1)$.
Full details are in the Appendix.

\subsection{Conducting sensitivity analyses}
\label{sec:sensitivity}

Using this model, we develop a practical protocol for sensitivity analysis to understand how epistemic uncertainty influences confidence in the estimated rankings.
In the absence of strong prior knowledge about LLM judge behavior, we recommend the following:
\ifsubmit
\vspace{-0.25cm}
\fi
\begin{enumerate}[itemsep=0pt, parsep=0pt]
    \item Start with conducting Bayesian inference with hyperparameters $\omega = 0$ (perfect constancy) and $\beta_{\max} = 5$ (moderate judge quality).
    \item Estimate how rankings shift as the constancy assumption is relaxed, by increasing $\omega$ from 0 to 8 while holding $\beta_{\max}$ fixed.
    \item Estimate how rankings shift as our prior over judge quality varies, by varying $\beta_{\max}$ from 0 to 20 while holding $\omega = 0$.
\end{enumerate}
\ifsubmit
\vspace{-0.2cm}
\fi
If rankings remain stable across parameter variations, the conclusions are reliable despite epistemic uncertainty.
Otherwise, one may need to inject additional prior knowledge (e.g., narrow the range of plausible hyperparameter values) or obtain gold-standard labels.

\section{EXPERIMENTS}

\begin{figure}
    \centering
    \includegraphics[width=0.95\linewidth]{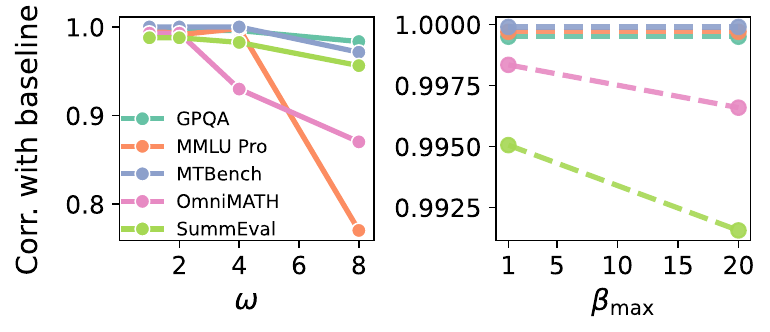}
      \includegraphics[width=0.47\linewidth]{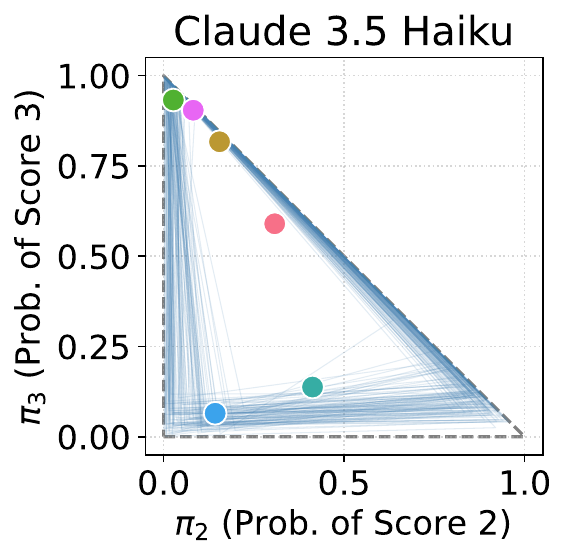}
  \includegraphics[width=0.46\linewidth]{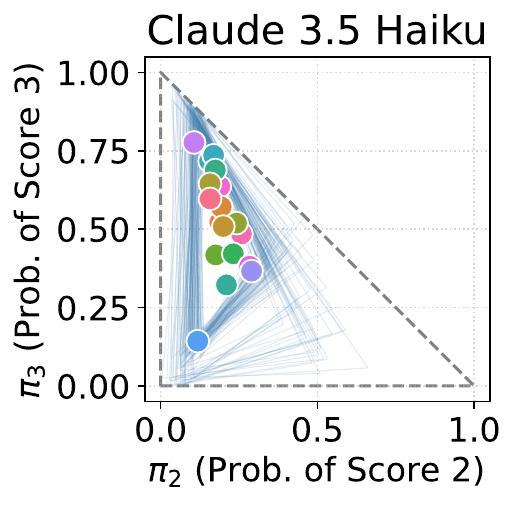}
  \ifsubmit
  \vspace{-0.2cm}
  \fi
    \caption{
    Top: Sensitivity of estimated rankings when varying random effects (RE, left) and judge quality (right) hyperparameters, by plotting the correlation between the estimated ranking for each hyperparameter and their base values ($\omega =0, \beta_{\max} =0)$.
    Bottom: Candidates visualized on the probability simplex with judge configurations sampled from posterior (blue triangles), for MTBench (left) and Omni-MATH (right).
    We collapse 5-level ratings to 3 levels for visualization (see Appendix for mapping details).
    }
    \ifsubmit
    \vspace{-0.2cm}
    \fi
    \label{fig:vary_params}
\end{figure}


Here we present experiments across various LLM benchmarks to assess the impact of epistemic uncertainty on ranking.
First, using the geometric Bayesian framework, we conduct sensitivity analyses to determine the robustness of learned rankings.
Second, we compare the overall performance of the Bayesian framework to existing methods in terms of coverage and correlation with ground truth rankings.
We present our main experimental results below, with ablation studies, implementation details, and extended analyses in the Appendix.

We evaluate the Bayesian framework across three categories of benchmark datasets:
\textbf{(i) Verifiable tasks with 2-level scores}: GPQA \citep{Rein2023-dv} and MMLU Pro~\citep{Wang2024-ic} contain multiple-choice questions graded as correct/incorrect, with judges having an abstention option.
\textbf{(ii) Multi-level human-judged tasks}: MTBench~\citep{Zheng2023-wr} evaluates multi-turn conversations while SummEval~\citep{Volske2017-bv} assesses summarizations. Both datasets are assessed on multi-level Likert scales, by LLMs as well as human experts.
 \textbf{(iii) ``Semi-verifiable'' tasks}:
Omni-MATH~\citep{Gao2024-ec} contains mathematical reasoning problems with reference solutions, though there is no single ground-truth answer.
LLM judges are asked to evaluate answers on a 3-level scale (correct, partial credit, incorrect). 

We use a two-stage protocol to compare ranking methods.
First, LLM judges evaluate candidates without access to ground truth, mirroring real-world usage. Second, we generate ground-truth scores for each answer by comparing against the correct multiple-choice answer on verifiable tasks and obtaining human-assigned scores for human-judged tasks.
For semi-verifiable tasks, we rescore each candidate's answer by providing the LLM judge with the provided reference answer.
For GPQA, MMLU Pro, and Omni-MATH, we evaluated 19 contemporary models including Claude, GPT, Gemini, Llama, Mistral, and Qwen variants, with Claude 3.5 Haiku and GPT-4o mini serving as judges.
These judges were deliberately selected to assess the performance of ranking methods when the judges are imperfect.
For MTBench and SummEval, we utilize their existing candidates and judges.

\subsection{Sensitivity analyses}
\label{sec:exp_sensitivity}

\begin{figure}
    \centering
    \includegraphics[width=0.65\linewidth]{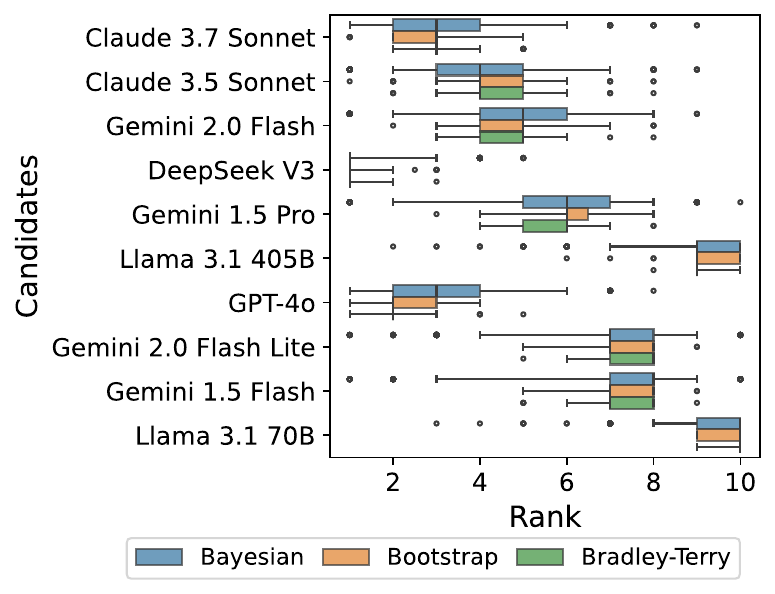}
    \ifsubmit
    \vspace{-0.2cm}
    \fi
    \caption{Estimated rankings for the top 10 candidates on GPQA, presented in order by their true rankings. Unlike existing methods, Bayesian method fully incorporates epistemic uncertainty.}
    \label{fig:GPQA_boxplots}
\end{figure}

\begin{table*}
    \centering
    \caption{Ranking method performance on benchmarks in terms of correlation and coverage of ground-truth ranks.
    }
    \footnotesize
    \resizebox{\textwidth}{!}{%
    \begin{tabular}{l|cc|cc|cc|cc|cc}
    \toprule
    Method & \multicolumn{2}{c}{GPQA} & \multicolumn{2}{c}{MMLU} &  \multicolumn{2}{c}{Omni-MATH} &  \multicolumn{2}{c}{SummEval} &  \multicolumn{2}{c}{MTBench}\\
    & Corr & Cov & Corr & Cov & Corr & Cov & Corr & Cov & Corr & Cov \\
    \midrule
    Bayesian & \textbf{\ci{0.916}{0.87}{0.94}} & \textbf{\ci{0.889}{0.72}{0.94}} & \textbf{\ci{0.940}{0.85}{0.95}} & \textbf{\ci{1.000}{0.90}{1.00}} & \textbf{\ci{0.791}{0.60}{0.83}} & \textbf{\ci{0.737}{0.37}{0.79}} & \textbf{\ci{0.888}{0.76}{0.91}} & \textbf{\ci{0.917}{0.67}{0.95}} & \textbf{\ci{1.000}{0.94}{1.00}} & \textbf{\ci{1.000}{1.00}{1.00}} \\[0.5em]
    Bootstrap & \ci{0.922}{0.88}{0.93} & \ci{0.556}{0.51}{0.83} & \ci{0.884}{0.84}{0.90} & \ci{0.474}{0.49}{0.68} & \ci{0.767}{0.71}{0.80} & \ci{0.368}{0.37}{0.47} & \ci{0.885}{0.78}{0.90} & \ci{0.729}{0.65}{0.96} & \ci{0.943}{0.94}{1.00} & \ci{1.000}{1.00}{1.00} \\[0.5em]
    Bradley-Terry & \ci{0.901}{0.88}{0.92} & \ci{0.389}{0.39}{0.67} & \ci{0.886}{0.84}{0.90} & \ci{0.474}{0.38}{0.58} & \ci{0.767}{0.72}{0.79} & \ci{0.368}{0.37}{0.42} & \ci{0.902}{0.82}{0.91} & \ci{0.771}{0.61}{0.91} & \ci{0.943}{0.94}{1.00} & \ci{1.000}{0.67}{1.00} \\[0.5em]
    Simple Average & \ci{0.922}{0.88}{0.93} & \ci{0.130}{0.02}{0.28} & \ci{0.884}{0.84}{0.90} & \ci{0.211}{0.14}{0.28} & \ci{0.767}{0.71}{0.80} & \ci{0.228}{0.18}{0.33} & \ci{0.885}{0.78}{0.90} & \ci{0.271}{0.09}{0.54} & \ci{0.943}{0.94}{1.00} & \ci{0.667}{0.67}{1.00} \\[0.5em]
    Single Judge & \ci{0.869}{0.81}{0.90} & \ci{0.130}{0.02}{0.28} & \ci{0.855}{0.81}{0.89} & \ci{0.211}{0.14}{0.28} & \ci{0.725}{0.70}{0.75} & \ci{0.228}{0.18}{0.33} & \ci{0.787}{0.58}{0.86} & \ci{0.285}{0.06}{0.45} & \ci{0.971}{0.94}{1.00} & \ci{0.833}{0.67}{1.00} \\
    \bottomrule
    \end{tabular}
    }
    \label{tab:non_stratified_results}
\end{table*}

For each benchmark dataset, we conducted Bayesian inference with random effects hyperparameter $\omega$ values of 0, 1, 2, 4, and 8, corresponding to expected random effect magnitudes of approximately 0, 0.5, 0.7, 0.8, and 0.9 respectively.
Figure~\ref{fig:vary_params} (left) shows that rankings remain remarkably stable across these variations for most datasets.
The correlation between rankings at different $\omega$ values and the baseline rankings exceeds 0.95 for GPQA, MTBench, and SummEval, indicating that significant violations of the constancy assumption have minimal impact on the resulting rankings.
Similarly, varying the judge quality hyperparameter $\beta_{\max}$ from 1 to 20 produces negligible effects on rankings across all datasets (Figure~\ref{fig:vary_params} right), with correlations consistently above 0.99.
This insensitivity to the judge quality prior suggests that the observed score data are sufficiently informative to overwhelm the prior on judge discrimination, making the resulting rankings robust to prior beliefs about $\beta_{\max}$.
However, both MMLU Pro and Omni-MATH show greater sensitivity to the constancy assumption, with correlations between baseline rankings dropping to 0.77 and 0.86 respectively when $\omega$ is increased from 0 to 8.
This sensitivity appears pronounced for MMLU Pro for $\omega > 4$, highlighting uncertainty in its rankings because of potential variability in judge performance across candidates.

To understand why estimated rankings are more or less robust, we can visualize the candidates in the probability simplex against possible judge configurations sampled from the posterior distribution of the Bayesian method.
Here we compare two contrasting datasets to illustrate why some benchmarks are robust to epistemic uncertainty while others are not.
As shown in Figure~\ref{fig:vary_params} left, estimated rankings for MTBench are robust even when the constancy assumption is relaxed because (i) candidates are well-separated across the simplex, with clear performance tiers and (ii) the posterior judge configurations are relatively consistent with similar triangular shapes with minimal variability.
In contrast, Figure~\ref{fig:vary_params} right shows that estimated rankings for Omni-MATH are more sensitive to epistemic uncertainty because (i) candidates cluster tightly in the central region with overlapping score distributions, making their relative ordering sensitive to judge vertex positions, and (ii) the posterior samples demonstrate considerable uncertainty in the judge configurations.
Simplex visualizations for MMLU Pro and GPQA are provided in the Appendix.

\subsection{Comparative performance evaluation}

While Section~\ref{sec:exp_sensitivity} examined ranking estimates qualitatively, this section will evaluate ranking estimates from the Bayesian framework quantitatively, in terms of Spearman correlation between estimated and ground-truth rankings (ranking accuracy) and coverage rate of 95\% credible or confidence intervals (calibration).
To produce a single point estimate from our Bayesian framework rather than a sensitivity curve, we place hyperpriors over $\omega$ and $\beta_{\max}$, integrating out epistemic uncertainty during posterior inference.
We compare against the following four baselines: single judge scoring (the score from one LLM judge, without adjudication), simple averaging across multiple judges, bootstrap confidence intervals~\citep{Xie2009-wv}, and Bradley-Terry pairwise comparison models~\citep{Bradley1952-kj}.

The Bayesian framework consistently achieves higher correlations and better coverage across all benchmark datasets (Table~\ref{tab:non_stratified_results}), with coverage rates improving by over 20 percentage points in some cases.
This dramatic improvement in coverage reflects our theoretical finding that epistemic uncertainty is essential for proper uncertainty quantification.
The correlation improvements, while more modest, demonstrate that incorporating judge uncertainty also benefits ranking estimates.
Figure~\ref{fig:GPQA_boxplots} illustrates how this coverage improvement is achieved: as seen on the GPQA dataset, the intervals from Bayesian inference are indeed significantly wider than those produced by comparator methods.
These wider intervals reflect genuine epistemic uncertainty about judge quality. Baseline methods produce narrower intervals by treating judge parameters as known, which leads to systematic undercoverage when those assumptions are violated (Table~\ref{tab:non_stratified_results}).
Conversely, credible intervals that are narrow represent areas of low posterior uncertainty, which can occur due to the discrete nature of ranks as well as the consistency in ranks assigned by judges.
Posterior distributions for the other datasets are given in the Appendix.

\section{DISCUSSION}
This work presents a geometric framework for understanding when gold-standard rankings can and cannot be recovered from only scores assigned by imperfect LLM judges and not gold-standard scores.
By representing judges and candidates on a probability simplex, we establish theoretical conditions for ranking identifiability and develop a sensitivity analysis approach.
Our key conclusion is that epistemic uncertainty about judge quality critically affects ranking reliability and is not captured by existing methods.
The geometric perspective reveals that identifiability depends on scoring levels, judge consistency, and candidate separation in the simplex.


Future work includes deriving formal guarantees on the maximum achievable ranking accuracy given judge quality constraints, and extending the framework to more diverse judge ensembles and scoring rubrics.
In addition, the Bayesian framework can be readily extended to incorporate gold-standard labels when they are available, as a Bayesian extension of methods like prediction-powered inference~\citep{Angelopoulos2023-tv}.
Finally, more work can be done to understand the tradeoffs between different methods for improving ranking estimates, such as through statistical correction in this work versus the collection of a small set of gold-standard scores; we include basic experiments comparing the two options in the Appendix.


\section*{ACKNOWLEDGMENTS}
J.F. and P.V. were supported through a Patient-Centered Outcomes Research Institute\textsuperscript{\textregistered} (PCORI\textsuperscript{\textregistered}) Award (ME-2022C1-25619).
The authors thank Gene Pennello and Nicholas Petrick (U.S. Food and Drug Administration, Center for Devices and Radiological Health) for helpful discussions.

\bibliography{paperpile_llm_judge, main}
\bibliographystyle{plainnat}

\newpage

%
\runningtitle{Supplementary Materials}

%
\runningauthor{Vossler, Xia, Mai, Subbaswamy, Feng}

\onecolumn
\aistatstitle{LLMs Judging LLMs: A  Simplex Perspective\\ \textit{Supplementary Materials}}

\appendix
\renewcommand\thefigure{\thesection.\arabic{figure}}

\section{GEOMETRIC PERSPECTIVE: ADDITIONAL DETAILS}
\label{app:geometric-details}

\subsection{Expected scores map to height in augmented space}
\label{app:vertical-projection}

A key geometric insight used throughout the analysis is that a candidate's expected score corresponds to its height when the probability simplex is embedded in a higher-dimensional space.
This result connects the abstract barycentric coordinates to the ranking task.

\begin{theorem}[Height Correspondence in Augmented Simplex]
\label{thrm:vertical-projection}
Consider embedding the $(M-1)$-dimensional probability simplex into $M$-dimensional space by adding a ``score'' axis.
Lift each judge vertex $\theta_{m}^{(j)}$ to position $(\theta_{m}^{(j)}, m)$, placing it at height $m$.
For any candidate $k$ with barycentric coordinates $\vec{\pi}_k = (\pi_{k,1}, \ldots, \pi_{k,M})$, when positioned in the augmented space using these same coordinates, its height equals its expected score $\mathbb{E}[S_k^*]$.
\end{theorem}

\begin{proof}
The candidate's position in the original $(M-1)$-dimensional simplex is determined by its barycentric coordinates:
\begin{align}
\gamma_k^{(j)} = \sum_{m=1}^M \pi_{k,m} \theta_m^{(j)}
\end{align}
where $\pi_{k,m}$ represents the weight given to judge vertex $\theta_m^{(j)}$.

To augment the space, the simplex is embedded into $M$-dimensional space by adding a ``score'' coordinate.
Each judge vertex $\theta_m^{(j)}$ is lifted to position $(\theta_m^{(j)}, m)$, placing it at height equal to its score value $m$.

When the candidate is positioned in this augmented space using the same barycentric coordinates $\pi_{k,m}$, the result is:
\begin{align}
\text{Augmented position} = \sum_{m=1}^M \pi_{k,m} \left(\theta_m^{(j)}, m\right)
\end{align}

This convex combination gives:
\begin{itemize}
\item Simplex coordinates: $\gamma_k^{(j)} = \sum_{m=1}^M \pi_{k,m} \theta_m^{(j)}$ (unchanged)
\item Height coordinate: $h_k = \sum_{m=1}^M \pi_{k,m} \cdot m$
\end{itemize}

The height $h_k = \sum_{m=1}^M \pi_{k,m} \cdot m$ is precisely the definition of the expected score $\mathbb{E}[S_k^*]$, where $\pi_{k,m} = \Pr(S_k^* = m)$.

Thus, the candidate's height in the augmented space directly encodes its expected score, so ranking candidates reduces to comparing their heights.
\end{proof}

This geometric correspondence means that candidates automatically acquire heights equal to their expected scores when the simplex is embedded in augmented space using the barycentric coordinates.
Ranking then reduces to comparing heights.
This representation also captures how uncertainty in judge vertex positions (epistemic uncertainty) translates to uncertainty in candidate heights and thus rankings.

\subsection{Recoverability is dataset-specific}

\begin{figure}
    \centering
    \includegraphics[width=0.2\linewidth]{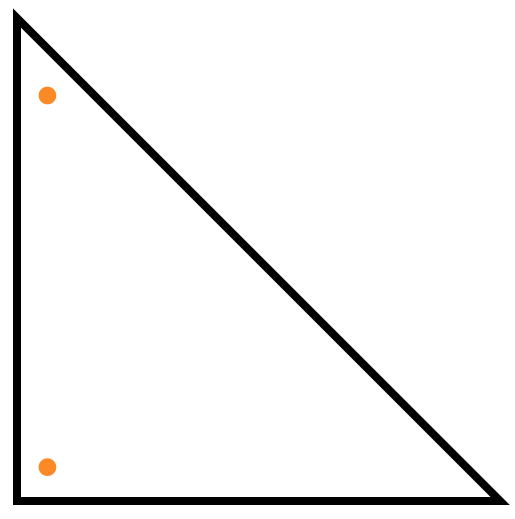}
    \caption{Example where candidate rankings are likely identifiable even when there are 3+ levels, due to characteristics of the dataset. Candidates have assigned scores that are very different (mostly 1's versus mostly 3's). 
    }
    \label{fig:identifiability_reasons}
\end{figure}

Special cases exist where rankings are robust to judge quality assumptions despite the general non-identifiability result for 3+-level scoring systems.
When candidates have substantially different score distributions---manifesting as widely separated points on the probability simplex---their relative ordering remains stable across a broad range of plausible judge configurations.
See Figure~\ref{fig:identifiability_reasons} as an example (Example 2 in Section 4.3).

\section{BAYESIAN MODEL}
\label{app:bayesian-model-detailed}

\subsection{Probability Model}
\label{app:bayesian-model}
For judge $j$ evaluating candidate $k$'s answer to the $i$-th question, the assigned score $\hat{S}_{ik}^{(j)}$ given its true score $S_{ik}^*$ is assumed to follow a multinomial distribution with parameter $\theta_{S_{ik}^*,k}^{(j)}$ (Figure~\ref{fig:plate}).
To keep the model tractable while preserving the geometric structure from Section~\ref{sec:geometric-framework}, conditional independence across judges and questions is assumed.
After marginalizing over the true latent scores, the likelihood of the observed data becomes:
\begin{align}
    \prod_{i=1}^{n}
    \prod_{j=1}^J 
    \prod_{k \ne j}
    \Big[\sum_{m=1}^M 
    \underbrace{\Pr\left(\hat{S}^{(j)}_{ik}|S_{ik}^* = m; \theta_{m,k}^{(j)}\right)}_{=\theta_{m,k,\hat{S}^{(j)}_{ik}}^{(j)}}
    \underbrace{\Pr(S_{ik}^* = m)}_{=\pi_{k,m}}
    \Big],
    \label{eq:obs_likelihood_marg_appendix}
\end{align}
where $n$ is the number of questions and self-evaluations ($k\ne j$) are excluded to avoid self-preference bias.

\begin{figure}
    \centering
    \includegraphics[width=0.5\linewidth]{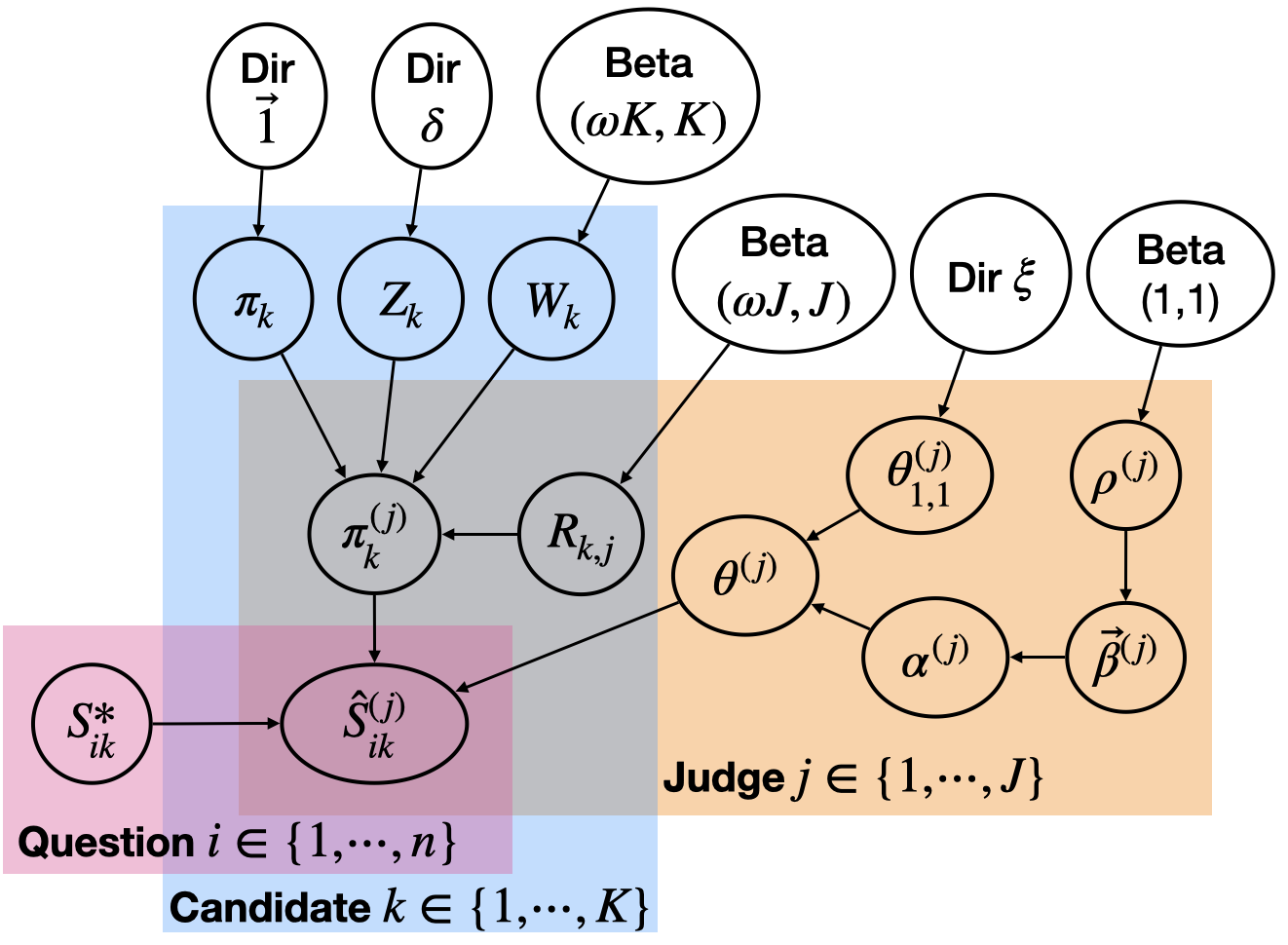}
    \caption{Plate diagram of the Bayesian model}
    \label{fig:plate}
\end{figure}

Although judge scores are likely correlated in practice and \eqref{eq:obs_likelihood_marg_appendix} simplifies the true data-generating mechanism, modeling the full joint distribution would add complexity and require additional constancy assumptions without changing the identifiability results in Section~\ref{sec:geometric-framework}.
For binary scoring systems, posterior inference under this model yields consistent estimators for prevalences and rankings even when correlation structure is ignored.

\subsection{Random Effects Parameterization}
\label{app:random-effects}

\subsubsection{Complete Hierarchical Specification}
To relax the constancy assumption while keeping the model tractable, random effects are introduced that allow judge-specific and candidate-specific deviations from base prevalences.
The complete hierarchical model is:

\begin{align}
Z_k &\sim \text{Dirichlet}(\delta) \quad \text{(candidate-specific random direction)}\\
R_{j} &\sim \text{Beta}(\omega J, J) \quad \text{(judge-specific random effect magnitude)}\\
W_{k} &\sim \text{Beta}(\omega K, K) \quad \text{(candidate-specific random effect magnitude)}\\
\vec\pi_{k}^{(j)} &= (1 - W_{k} R_{j}) \vec\pi_k + W_{k} R_{j} Z_k \quad \text{(perturbed prevalences)}
\end{align}

The perturbed prevalences $\vec\pi_{k}^{(j)}$ replace $\vec\pi_{k}$ in the likelihood \eqref{eq:obs_likelihood_marg_appendix}. 
This parameterization ensures that $\vec\pi_{k}^{(j)}$ remains on the probability simplex for all valid parameter values.

\subsubsection{Choice of Prevalence vs. Confusion Matrix Perturbation}
Two approaches for introducing random effects were considered: perturbations in judge performance (i.e. $\theta_{m,k}^{(j)}$ as perturbations of $\theta_{m}^{(j)}$) or in score prevalences (i.e. $\vec\pi_k^{(j)}$ as perturbations of $\vec\pi_k$).

Prevalence perturbation was chosen for several reasons. For $M$-level scoring with $J$ judges and $K$ candidates, confusion matrix perturbation requires $O(JKM^2)$ parameters while prevalence perturbation requires only $O(JKM)$.
Fewer parameters reduce the risk of overfitting and model misspecification and enable faster MCMC convergence.
Deviations in prevalences also correspond directly to changes in perceived candidate quality across judges.

\subsubsection{Setting $\delta$ for Detecting Specific Biases}
The Dirichlet parameter $\delta$ controls the distribution of random directions and can be tuned to detect specific types of bias:

\begin{itemize}
    \item Uniform exploration: $\delta = [1, 1, ..., 1]$ explores all directions equally
    \item Score inflation detection: For 3-level scoring, $\delta = [1, 4, 10]$ places more weight on higher scores, detecting judges who systematically overrate candidates
    \item Score deflation detection: For 3-level scoring, $\delta = [10, 4, 1]$ places more weight on lower scores, detecting overly critical judges
    \item Central tendency bias: $\delta = [1, 10, 1]$ for 3-level scoring detects judges who avoid extreme scores
\end{itemize}

In the experiments, score inflation detection ($\delta = [1, 4, 10]$) is used because self-preference bias typically manifests as inflated scores for one's own model family.

\subsection{Judge Quality Prior Specification}
\label{app:judge-quality-prior}

The judge quality prior is implemented through a transition weight framework that parameterizes the confusion matrix columns (judge vertices) as weighted combinations of base positions.
This section provides the full mathematical specification of the framework briefly introduced in Section~\ref{sec:judge_quality} of the main paper.
This approach encodes beliefs about relative vertex positions while maintaining the probabilistic constraints of the confusion matrix.

\subsubsection{Transition Weight Framework}

The judge's confusion matrix columns (vertices) are parameterized through a directed acyclic graph where nodes represent confusion matrix entries and edges control how probability mass flows between them.
Prior beliefs about judge quality are encoded through the edge weights while ensuring mathematical constraints are satisfied.

\begin{figure}
  \centering
  \begin{minipage}[t]{0.2\textwidth}
    \vspace{0pt}
    \centering
    \includegraphics[width=\linewidth]{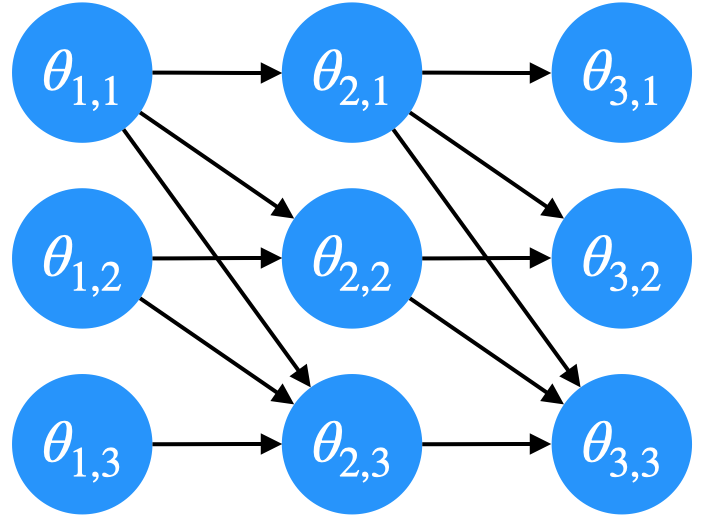}
  \end{minipage}
  \hfill
  \begin{minipage}[t]{0.6\textwidth}
    \vspace{0pt}
    \begin{align*}
    &\left\{\alpha_{(m_1, m_2) \rightarrow  (m_1', m_2')} \right\}_{(m_1', m_2') \text{ s.t. } (m_1, m_2) \rightarrow (m_1', m_2')}  \sim \mathrm{Dirichlet}(\vec{\beta}_{(m_1, m_2)})\\
    &\theta_{m_1',m_2'}^{(j)} = 
      \sum_{(m_1, m_2) \rightarrow (m_1',m_2') } \theta_{(m_1, m_2)}^{(j)} \alpha_{(m_1, m_2) \rightarrow (m_1',m_2')}
    \end{align*}
  \end{minipage}
  \hfill
\caption{
Transition weight framework for encoding judge quality priors in 3-level scoring.
Nodes represent pairs $(m_1, m_2)$ where $m_1$ is the true score and $m_2$ is the assigned score.
Edges show allowed transitions, with weights $\alpha$ drawn from Dirichlet priors parameterized by $\vec{\beta}_{(m_1, m_2)}$ (see Appendix~\ref{app:judge-quality-prior} for formal specification).
}
\
\end{figure}

The formal specification of the transition weights is:
\begin{align*}
    &\left\{\alpha_{(m_1, m_2) \rightarrow  (m_1', m_2')} \right\}_{(m_1', m_2') \text{ s.t. } (m_1, m_2) \rightarrow (m_1', m_2')} \sim \text{Dirichlet}(\vec{\beta}_{(m_1, m_2)})\\
    \theta_{m_1',m_2'}^{(j)} &= \sum_{(m_1, m_2) \rightarrow (m_1',m_2') } \theta_{(m_1, m_2)}^{(j)} \alpha_{(m_1, m_2) \rightarrow (m_1',m_2')}
\end{align*}

Each node $(m_1, m_2)$ in the transition graph represents a confusion matrix entry: the probability of assigning score $m_2$ when the true score is $m_1$.
The transition weights $\alpha_{(m_1, m_2) \rightarrow (m_1', m_2')}$ control how probability mass flows from parent to child nodes.
All outgoing weights from any parent node sum to one, ensuring the result remains a valid probability distribution.
The final confusion matrix entry $\theta_{m_1',m_2'}^{(j)}$ for judge $j$ is a weighted average of its parent nodes' values, weighted by the incoming edge weights.

For example, in Figure~\ref{fig:transition}, the probability $\theta_{2,2}$ (assigning score 2 when true score is 2) receives contributions from both $\theta_{1,1}$ and $\theta_{1,2}$ through their respective edge weights.

\subsubsection{Parameterization for 3-Level Scoring}

For a 3-level scoring system, the Dirichlet parameters $\vec{\beta}_{(m_1, m_2)}$ are specified as follows:
\begin{align*}
    \vec{\beta}_{1,1}^{(j)} &= [1, 1 + \rho^{(j)} \beta_{\max}, 1]\\
    \vec{\beta}_{1,2}^{(j)} &= [{\beta}_{1,1,2}^{(j)}, {\beta}_{1,1,3}^{(j)}]\\
    \vec{\beta}_{2,1}^{(j)} &= [1, 1, 1 + \rho^{(j)} \beta_{\max}]\\
    \vec{\beta}_{2,2}^{(j)} &= [{\beta}_{2,1,2}^{(j)}, {\beta}_{2,1,3}^{(j)}],
\end{align*}
where:
\begin{itemize}
    \item $\beta_{\max}$ is the global hyperparameter controlling prior strength about judge quality
    \item $\rho^{(j)} \sim \text{Beta}(1,1)$ allows judge-specific variation in discrimination ability (drawn from a uniform hyperprior on $[0,1]$)
    \item The structure ensures higher weight on correct identification (diagonal elements) as $\beta_{\max}$ increases
\end{itemize}

Note that $\alpha_{1,3}$ and $\alpha_{2,3}$ are deterministically set equal to one as these represent terminal nodes in the transition graph with no outgoing edges.

Intuitively, when $\beta_{\max} = 0$, all Dirichlet parameters equal 1 (uniform prior), allowing maximum flexibility in judge behavior.
As $\beta_{\max}$ increases, the parameters $[1, 1 + \rho^{(j)} \beta_{\max}, 1]$ place more weight on the middle transition, encoding our belief that better judges more consistently assign higher scores when the true score increases.

\subsubsection{Enforcing Monotonicity Through Graph Structure}

The transition graph topology automatically enforces Assumption~\ref{assumption:no_stupid} (monotonicity) by restricting which edges exist. Specifically, edges from node $(m_1, m_2)$ to $(m_1', m_2')$ only exist when:
\begin{itemize}
    \item $m_1' = m_1 + 1$ (moving to the next true score level)
    \item $m_2' \geq m_2$ (assigned score cannot decrease as true score increases)
\end{itemize}

This structural constraint ensures that judges cannot be ``label-flippers'' who systematically assign lower scores to better answers, without requiring explicit inequality constraints during MCMC sampling.

\subsubsection{Extension to $M$-Level Scoring}

For general $M$-level scoring systems, the framework extends naturally:
\begin{itemize}
    \item The graph has $M(M+1)/2$ nodes representing all valid $(m_1, m_2)$ pairs where $m_2 \geq m_1$
    \item Edge constraints follow the same monotonicity rules
    \item Dirichlet parameters follow the pattern: higher values along paths that preserve or increase the assigned score relative to the true score
\end{itemize}

The parameterization for $M > 3$ follows similar principles, with $\beta_{\max}$ controlling the concentration around correct identification. The accompanying code provides exact implementation details for different values of $M$.

\section{IMPLEMENTATION DETAILS}
\label{app:implement}

\textbf{Bayesian inference}: Posterior inference is conducted using Hamiltonian Monte Carlo (HMC) in Stan~\citep{Stan_Development_Team2021-id}.
HMC ran with 4 chains, each with 1000 warmup iterations and 1000 sampling iterations.



\textbf{Abstention}: In practice, the judge may not always be sure what score to assign.
The judge is given the option to abstain when there are only two levels, e.g. correct or not.

\section{EXPERIMENT CONFIGURATION}

\subsection{Benchmarks}
\begin{itemize}
    \item \textbf{GPQA} \citep{Rein2023-dv}: Questions resistant to simple internet searches across STEM domains, stratified by difficulty (undergraduate, graduate, post-graduate).
    
    \item \textbf{MMLU Pro} \citep{Wang2024-ic}: Enhanced professional knowledge questions across 16 domain-specific subcategories from natural sciences, social sciences, and humanities.

    \item \textbf{MTBench} \citep{Zheng2023-wr}: A conversational benchmark evaluating single- and multi-turn dialogue capabilities across diverse scenarios (creative writing, reasoning, coding, mathematics, role-playing). Human judges rated responses on a 10-point scale.
    LLM judges are asked for ratings on a simplified 5-point scale.
    
    \item \textbf{TLDR} (aka SummEval) \citep{Fabbri2021-ie}: A summarization benchmark where models condense news articles into concise summaries, with human ratings across four dimensions (relevance, consistency, fluency, coherence) on 5-point scales.

    \item \textbf{Omni-MATH} \citep{Gao2024-ec}: A benchmark of competition-level problems from International and National Olympiads. These problems are difficult to evaluate automatically because solutions vary in approach, notation, and presentation; multiple valid solution paths may exist; and partial correctness must be assessed along multiple dimensions.
\end{itemize}

\subsection{LLM judges}
\label{app:llm_judges}

Two LLM judges are used: Claude 3.5 Haiku (anthropic/claude-3-5-haiku-20241022) and GPT-4o Mini (gpt-4o-mini-2024-07-18) for all benchmark datasets except TLDR.
For TLDR, the provided judge-assigned scores from older LLMs are used: GPT-4-0314, GPT-3.5-Turbo-0301, and Llama-2-70b-chat-hf.

To mitigate position bias, the order of candidate responses presented to judges was randomized.

\subsection{LLM candidates}
\label{app:llm_candidates_list}

For consistent comparison across datasets, the set of candidate LLMs shown in Table~\ref{tab:llm_candidates} was evaluated.

\begin{table}[ht]
\centering
\caption{LLM candidates evaluated across all experimental settings}
\label{tab:llm_candidates}
\adjustbox{center}{
\begin{tabular}{llll}
\toprule
\textbf{Model Family} & \textbf{Model Name} & \textbf{Version/Date} & \textbf{Experiments} \\
\midrule
Anthropic & Claude 3.5 Haiku & claude-3-5-haiku-20241022 & GPQA, MMLU Pro, Omni-MATH \\
Anthropic & Claude 3.5 Sonnet & claude-3-5-sonnet-20241022 & GPQA, MMLU Pro, Omni-MATH \\
Anthropic & Claude 3.7 Sonnet & claude-3-7-sonnet-20250219 & GPQA, MMLU Pro, Omni-MATH \\
DeepSeek & DeepSeek V3 & deepseek-v3 & GPQA \\
Google & Gemini 1.5 Flash & gemini-1.5-flash-002 & GPQA, MMLU Pro, Omni-MATH \\
Google & Gemini 1.5 Pro & gemini-1.5-pro-002 & GPQA, MMLU Pro, Omni-MATH \\
Google & Gemini 2.0 Flash & gemini-2.0-flash-001 & GPQA, MMLU Pro, Omni-MATH \\
Google & Gemini 2.0 Flash Lite & gemini-2.0-flash-lite-preview-02-05 & GPQA, MMLU Pro, Omni-MATH \\
Meta & Llama 3.1 405B & llama-3.1-405b-instruct-turbo & GPQA, MMLU Pro, Omni-MATH \\
Meta & Llama 3.1 70B & llama-3.1-70b-instruct-turbo & GPQA, MMLU Pro, Omni-MATH \\
Meta & Llama 3.1 8B & llama-3.1-8b-instruct-turbo & GPQA, MMLU Pro, Omni-MATH \\
Meta & Llama 4 Maverick 17B & llama-4-maverick-17b-128e-instruct-fp8 & MMLU Pro, Omni-MATH \\
Meta & Llama 4 Scout 17B & llama-4-scout-17b-16e-instruct & MMLU Pro, Omni-MATH \\
Mistral AI & Mistral 7B & mistral-7b-instruct-v0.3 & GPQA \\
Mistral AI & Mixtral 8x22B & mixtral-8x22b-instruct-v0.1 & GPQA \\
Mistral AI & Mixtral 8x7B & mixtral-8x7b-instruct-v0.1 & GPQA \\
OpenAI & GPT-4.1 & gpt-4.1-2025-04-14 & MMLU Pro, Omni-MATH \\
OpenAI & GPT-4.1 mini & gpt-4.1-mini-2025-04-14 & MMLU Pro, Omni-MATH \\
OpenAI & GPT-4.1 nano & gpt-4.1-nano-2025-04-14 & MMLU Pro, Omni-MATH \\
OpenAI & GPT-4o & gpt-4o-2024-11-20 & GPQA, MMLU Pro, Omni-MATH \\
OpenAI & GPT-4o mini & gpt-4o-mini-2024-07-18 & GPQA, MMLU Pro, Omni-MATH \\
Qwen & Qwen 2.5 72B & qwen2.5-72b-instruct-turbo & GPQA, MMLU Pro, Omni-MATH \\
Qwen & Qwen 2.5 7B & qwen2.5-7b-instruct-turbo & GPQA, MMLU Pro, Omni-MATH \\
\bottomrule
\end{tabular}
}
\end{table}

For the MTBench dataset \citep{Zheng2023-wr}, the candidates in the provided dataset are assessed: GPT-4-0613, Claude-1, Llama-2-13B-Chat, Vicuna-13B, and Alpaca-13B.
Similarly, for the TLDR (SummEval) benchmark \citep{Fabbri2021-ie}, the 12 provided language models are assessed.

\subsubsection{Comparison Methods}
The Bayesian adjudication framework is compared against the following baseline methods:

\textbf{Simple Averaging:} This approach computes the mean score for each candidate across all evaluations and determines rankings based on these averages. It treats each judge's assessment with equal weight and assumes judge scores accurately reflect true performance.

\textbf{Single Judge Aggregation:} This approach collapses distinctions between multiple judges, treating all evaluations as if they came from a single judge. It computes the mean score for each candidate across all judge evaluations, ignoring judge identity.

\textbf{Simple Averaging with Bootstrap Confidence Intervals:} This comparator uses the bootstrap approach for population ranks proposed in \cite{Xie2009-wv} to generate confidence intervals for the simple averaging estimate.

\textbf{Pairwise Comparison Approach: } As a representative pairwise comparison method, an extension of the Bradley-Terry model with ties \citep{Rao1967-fv} is used, which estimates candidate ability parameters based on win-loss-tie patterns in pairwise evaluations. Confidence intervals are calculated using the bootstrap.

\textbf{Prediction-Powered Inference (PPI):} For benchmarks with ground truth labels, PPI \citep{Angelopoulos2023-tv} is implemented as an alternative ranking method. PPI uses a small labeled dataset to calibrate predictions from LLM judges on a larger unlabeled dataset, providing statistically valid confidence intervals for candidate rankings. The implementation: (i) randomly partitions questions into labeled and unlabeled sets (using 5\% or 10\% labeled fractions), (ii) applies the PPI mean estimator using the \texttt{ppi\_py} library with judge scores converted to a 0-1 scale, and (iii) generates confidence intervals through bootstrap resampling (300 iterations) while maintaining the fixed labeled/unlabeled partition. PPI is applied to GPQA, MMLU Pro, and Omni-MATH benchmarks where ground truth answers are available. Results for different labeled data fractions are presented in Section~\ref{app:ppi-results}.

Performance is assessed using Spearman's rank correlation with ground truth rankings when available, and coverage rates of 95\% credible/confidence intervals for uncertainty calibration. Together these metrics evaluate both point estimate accuracy and uncertainty calibration across methods.

\subsection{Prompts}
\label{app:prompts}

Below are the prompts used for evaluating candidate answers in each experiment. Where possible, structured generation is used to produce JSON outputs.

\subsubsection{Binary Verification Judge Prompt}
For the GPQA and MMLU Pro dataset experiments with multiple-choice questions and ground truth answers, the following prompt is used:

\begin{lstlisting}[breaklines]
You are evaluating candidate answers to a multiple-choice question.

- Consistency: How well the candidate's explanation aligns with their final multiple choice selection (1-5 scale).
  * 1 = The explanation contradicts the selected answer
  * 2 = Major disconnects between explanation and selected answer
  * 3 = Explanation partially supports the answer with some inconsistencies
  * 4 = Explanation mostly supports the answer with minor inconsistencies
  * 5 = Explanation perfectly aligns with and justifies the selected answer

- Accuracy: Did the candidate select the correct answer choice? (-1 = no, 1 = yes, 0 = unsure)
  * Provide a concise explanation referencing key facts or reasoning that makes the answer correct or incorrect

<QUESTION> 
[[question]]
</QUESTION>

[[candidates_section]]


Respond with a JSON object containing evaluations for all candidates and ensure that your JSON response:
- Uses the exact structure provided below
- Includes only the evaluation without additional preamble or commentary
- Properly escapes any special characters in the reasoning strings

{
  "evaluations" : [
    {
      "model_id": "1",
      "consistency": {
        "reasoning": str,
        "score": int
      },
      "accuracy": {
        "reasoning": str,
        "score": int
      }
    },
    ...
  ]
}
\end{lstlisting}

In the actual implementation, the \texttt{[[candidates\_section]]} placeholder is dynamically populated with candidate answers and explanations using the following format:

\begin{lstlisting}[breaklines]
<CANDIDATE #{i} ANSWER>
{candidate.get(answer_key)}
</CANDIDATE ANSWER>

<CANDIDATE #{i} EXPLANATION>
{rationale}
</CANDIDATE #{i} EXPLANATION>
\end{lstlisting}

where \texttt{\{i\}} is the candidate number, \texttt{\{candidate.get(answer\_key)\}} retrieves the candidate's multiple-choice answer, and \texttt{\{rationale\}} contains their explanation for that answer.

\subsubsection{Human Judgement Prompts}

\textbf{MTBench Judge Prompts}:\\
For MTBench, a prompt similar to the original one given to human judges in \citep{Zheng2023-wr} is used for the single-turn and two-turn cases. As with the binary tasks, the \texttt{[[question]]} and \texttt{[[candidates\_section]]} placeholders are replaced with the prompt and the conversation history using the following format:

\begin{lstlisting}[breaklines]
Please act as an impartial judge and evaluate the quality of the responses provided by AI assistants to the user question displayed below. Your evaluation should consider factors such as the helpfulness, relevance, accuracy, depth, creativity, and level of detail of their responses. Avoid any position biases and ensure that the order in which the responses were presented does not influence your decision. Do not allow the length of the responses to influence your evaluation. Do not favor certain names of the assistants.Be as objective as possible. Rate the response on a scale of 1 to 5 (1=Very Bad, 5=Very Good), along with the reasoning.

Initial User Question:
<PROMPT>
[[question]]
</PROMPT>

Assistant Conversation(s):
[[candidates_section]]

{
    "evaluations" : [
    {
        "model_id": "1",
        "overall": {
        "reasoning": str,
        "score": int
        }
    },
    {
        "model_id": "2",
        "overall": {
        "reasoning": str,
        "score": int
        }
    },
    {
        "model_id": "3",
        "overall": {
        "reasoning": str,
        "score": int
        }
    }
    ]
}
\end{lstlisting}
The candidate presentation format differs between single-turn and multi-turn evaluations. For single-turn interactions, only the initial response is presented:
\begin{lstlisting}[breaklines]
<CANDIDATE #{i}>
{response1}
</CANDIDATE #{i}>
\end{lstlisting}

For two-turn interactions, the complete conversation history is presented with delineation between turns:
\begin{lstlisting}[breaklines]
<CANDIDATE #{i}>

<TURN 1>
[User Prompt]
{prompt1}
[Assistant Response]
{response1}
</TURN 1>

<TURN 2>
[User Prompt]
{prompt2}
[Assistant Response]
{response2}
</TURN 2>
</CANDIDATE #{i}>
\end{lstlisting}

In cases where a candidate fails to respond to the second turn, this absence is explicitly noted:
\begin{lstlisting}[breaklines]
<TURN 2>
[No second turn response provided]
</TURN 2>
\end{lstlisting}

\textbf{TLDR (SummEval) Judge Prompts}:\\
For the TLDR (SummEval) dataset, the same evaluation framework as in \citep{Fabbri2021-ie} is used, assessing news article summarization quality across four criteria: relevance, consistency, fluency, and coherence. Each dimension is evaluated on a 5-point Likert scale with specific definitions to ensure consistent interpretation:

\begin{lstlisting}[breaklines]
Instructions: In this task you will evaluate the quality of summaries written for a news article. You will be shown the original article and [[num_candidates]] candidate summaries.

To correctly solve this task, follow these steps:

1. Carefully read the original news article provided below.
2. Read the candidate summaries presented in the <CANDIDATE #i ANSWER> sections.
3. Rate each summary on a scale from 1 (very low) to 5 (very high) based on its relevance, consistency, fluency, and coherence. Note that summaries that are very similar on an axis may receive the same score.

Definitions:
*   Relevance: The rating measures how well the summary captures the key points of the article. Summaries in which all and only the important aspects are contained will receive the highest rating.
*   Consistency: The rating measures whether the facts in the summary are consistent with the facts in the original article. The summary should stay true to the facts reported and not make up untrue information.
*   Fluency: This rating measures the quality of individual sentences: are they well-written and grammatically correct?
*   Coherence: This rating measures the quality of all sentences collectively: do they fit together and sound natural? Consider the quality of the summary as a whole.

Original news article:
[[question]]

Candidate Summaries:
[[candidates_section]]

Now provide your scores in the following JSON format. Ensure your response is a single JSON object, starting with {{ and ending with }}, and includes evaluations for all [[num_candidates]] candidates:

{{
    "evaluations": [
    // Evaluation for Candidate #1
    {{
        "model_id": "1", // Corresponds to Candidate #1
        "relevance": {{
        "reasoning": "Provide your reasoning for the relevance score here.",
        "score": int // Score from 1 to 5
        }},
        "consistency": {{
        "reasoning": "Provide your reasoning for the consistency score here.",
        "score": int // Score from 1 to 5
        }},
        "fluency": {{
        "reasoning": "Provide your reasoning for the fluency score here.",
        "score": int // Score from 1 to 5
        }},
        "coherence": {{
        "reasoning": "Provide your reasoning for the coherence score here.",
        "score": int // Score from 1 to 5
        }}
    }},
    // Add evaluations for Candidate #2, #3, ... up to #[[num_candidates]] following the same structure
    // Example for Candidate #2:
    /*
    {{
        "model_id": "2", // Corresponds to Candidate #2
        "relevance": {{
        "reasoning": "...",
        "score": int
        }},
        "consistency": {{
        "reasoning": "...",
        "score": int
        }},
        "fluency": {{
        "reasoning": "...",
        "score": int
        }},
        "coherence": {{
        "reasoning": "...",
        "score": int
        }}
    }}
    */
    // ... other candidates ...
    ]
}}
\end{lstlisting}

The \texttt{[[question]]} and \texttt{[[candidates\_section]]} placeholders are filled with the news article and the candidate summaries in the same format as the other datasets.

\subsubsection{Semi-verifiable task Judge Prompt}

For the Omni-MATH dataset experiments, a two-stage evaluation process is used to assess mathematical reasoning when multiple solution paths may be valid. This measures both standalone solution quality and alignment with reference solutions.

\textbf{Stage 1: Evaluation Without Ground Truth}:\\
In the first stage, judge LLMs evaluate candidate solutions based solely on mathematical correctness without access to reference answers, mimicking how human experts might evaluate mathematical work without preconceived notions of the ``correct'' approach. The prompt uses a 3-point accuracy scale (-1 for incorrect, 0 for partially correct, 1 for correct), with explicit instructions to use the middle category sparingly. The full prompt is:

\begin{lstlisting}[breaklines]
Instructions: Evaluate the quality of candidate answers to mathematical questions. You will be shown the original question and [[num_candidates]] candidate answers.

To correctly solve this task, follow these steps:

1. Carefully read the original question to understand what is being asked.
2. Read each candidate answer carefully.
3. Rate each answer according to the criteria below based on general mathematical knowledge and reasoning.
4. Provide clear justification for each score with specific references to the candidate's answer.

Rate each answer using the following criteria:

### Accuracy Assessment (1 for correct, 0 for partially correct/borderline, -1 for incorrect)
Based on your mathematical knowledge, how accurate is the candidate answer? Strive to categorize answers as either Correct (1) or Incorrect (-1). Reserve the Partially Correct/Borderline (0) score for answers that contain significant correct elements but also notable errors or omissions, making a definitive Correct/Incorrect judgment difficult, or for answers that are technically correct but incomplete in a way that affects the final conclusion.
* 1 (Correct): The answer is mathematically sound, reaches a valid conclusion, and is substantially free of errors.
* 0 (Partially Correct / Borderline): The answer contains significant correct elements but also notable errors or omissions preventing a clear "Correct" score OR the answer is technically correct but misses key steps or context, making it significantly less complete. Use this score sparingly.
* -1 (Incorrect): The answer contains significant mathematical errors or reaches an incorrect conclusion.

Question:
[[question]]

Candidates Summaries:
[[candidates_section]]

Respond with a JSON object containing evaluations for all candidates and ensure that your JSON response:
- Uses the exact structure provided below
- Includes only the evaluation without additional preamble or commentary
- Properly escapes any special characters in the reasoning strings
- Always output the reasoning before providing a final score

{
  "evaluations" : [
    {
      "model_id": "1",
      "accuracy": {
        "reasoning": str,
        "score": int
      }
    },
    ...
  ]
}
\end{lstlisting}

\textbf{Stage 2: Evaluation With Ground Truth Reference}:\\
In the second stage, judge LLMs re-evaluate candidate solutions with access to reference solutions, providing a benchmark for alignment with established approaches while still allowing alternative valid solution paths. The prompt maintains the same 3-point scale but refocuses evaluation on comparison with the reference solution. The full prompt is:

\begin{lstlisting}[breaklines]
Instructions: Evaluate the quality of candidate answers to mathematical questions. You will be shown the original question, the ground truth reference answer, and [[num_candidates]] candidate answers.

To correctly solve this task, follow these steps:

1. Carefully read the original question.
2. Carefully read the ground truth reference answer to understand the correct approach and solution.
3. For each candidate answer:
   - Read the entire response
   - Evaluate it against the ground truth reference answer
   - Score it according to the criteria below
   - Provide clear justification for each score with specific references to both the candidate answer and ground truth

Rate each answer using the following criteria relative to the ground truth reference answer:

### Accuracy Assessment (1 for correct, 0 for partially correct/borderline, -1 for incorrect)
Based on the reference answer, how accurate is the candidate answer? Strive to categorize answers as either Correct (1) or Incorrect (-1). Reserve the Partially Correct/Borderline (0) score for answers that contain significant correct elements but also notable errors or omissions, making a definitive Correct/Incorrect judgment difficult, or for answers that are technically correct but incomplete in a way that affects the final conclusion compared to the reference.
* 1 (Correct): The answer reaches the same mathematical conclusion as the reference answer (even if using a different valid approach) and is substantially free of errors.
* 0 (Partially Correct / Borderline): The answer contains significant correct elements but also notable errors or omissions preventing a clear "Correct" score OR the answer is technically correct but misses key steps or context provided in the reference, making it significantly less complete. Use this score sparingly.
* -1 (Incorrect): The answer reaches a different conclusion from the reference answer or contains significant mathematical errors that invalidate the result.

Question:
[[question]]

Ground Truth Reference Answer:
[[ground_truth_answer]]

Candidates Summaries:
[[candidates_section]]

Respond with a JSON object containing evaluations for all candidates and ensure that your JSON response:
- Uses the exact structure provided below
- Includes only the evaluation without additional preamble or commentary
- Properly escapes any special characters in the reasoning strings
- Always output the reasoning before providing a final score

{
  "evaluations" : [
    {
      "model_id": "1",
      "accuracy": {
        "reasoning": str,
        "score": int
      }
    },
    ...
  ]
}
\end{lstlisting}

\textbf{Combined Analysis}:\\
This two-stage approach enables separate analyses of intrinsic solution quality and reference alignment. The second-stage (reference-based) evaluations serve as pseudo-ground truth when comparing the Bayesian ranking methods against baselines. The first-stage evaluations indicate the judge's standalone mathematical reasoning ability---how often judges can identify correct solutions without reference answers. The gap between stage-one and stage-two evaluations may also provide a signal about problem difficulty and the capabilities of both candidate and judge models.

In the actual implementation, the \texttt{[[question]]}, \texttt{[[ground\_truth\_answer]]}, and \texttt{[[candidates\_section]]} placeholders are dynamically populated with the mathematical problem statement, reference solution, and candidate solutions, respectively. The candidate solutions are presented in the same format as in the other experimental settings.

\subsection{MTBench self-preference}
\label{app:mtbench-self-pref}
Table~\ref{tab:mtbench-self-pref} presents the frequency of scores (ranging from 1, low, to 5, high) assigned by the two LLM judges—Claude 3.5 Haiku and GPT-4o mini—to various candidate LLMs based on their responses to two-turn questions from the MTBench dataset. The results suggest self-preference bias in LLM-based evaluations. The Claude 3.5 Haiku judge awarded its predecessor, Claude v1, a high frequency of top scores (48 instances of `5'), surpassing other models like GPT-4 (40 instances of `5'). The GPT-4o mini judge assigned an equal number of perfect `5' scores (50 instances each) to both its own family model, GPT-4, and to Claude v1. While GPT-4o mini gives the same number of top scores to GPT-4 and Claude, the scores from the Claude judge suggest that models receive more favorable evaluations from judges within the same model family, indicating self-preference.

\begin{table}[ht]
\centering
\caption{Distribution of scores (1-5) assigned by Claude 3.5 Haiku and GPT-4o mini judges to candidate LLMs on the MTBench two-turn benchmark. Cell values represent the frequency of each score.}
\label{tab:mtbench-self-pref}
\begin{tabular}{lccccccccccc}
\toprule
 & \multicolumn{5}{c}{Claude 3.5 Haiku} & \multicolumn{5}{c}{GPT-4o mini} \\
\cmidrule(lr){2-6} \cmidrule(lr){7-11}
Model Name & 1 & 2 & 3 & 4 & 5 & 1 & 2 & 3 & 4 & 5 \\
\midrule
Llama 13B & 29 & 32 & 11 & 5 & 0 & 43 & 24 & 15 & 0 & 0 \\
Alpaca 13B & 3 & 33 & 33 & 8 & 3 & 17 & 25 & 29 & 10 & 0 \\
Vicuna 13B v1.2 & 1 & 7 & 24 & 27 & 19 & 7 & 12 & 19 & 25 & 18 \\
Claude v1 & 0 & 3 & 2 & 21 & 48 & 1 & 6 & 6 & 18 & 50 \\
GPT-3.5 Turbo & 0 & 2 & 11 & 38 & 20 & 1 & 7 & 14 & 23 & 37 \\
GPT-4 & 0 & 1 & 6 & 26 & 40 & 0 & 4 & 8 & 20 & 50 \\
\bottomrule
\end{tabular}
\end{table}

\subsection{Score Mapping for Simplex Visualizations}
\label{app:score_mapping}
For the probability simplex visualizations of 3+-level Likert scale datasets, the original 5-point scales are mapped to 3-point scales to enable more interpretable simplex representations.
The mapping functions $f: \{0,1,2,3,4,5\} \mapsto \{1,2,3\}$ are defined as follows:

\textbf{TLDR mapping:}
\begin{align*}
    f_{\text{TLDR}}(s) = 
    \begin{cases}
        1,& \text{if } s \in \{1,2\}\\
        2,& \text{if } s \in \{0,3,4\}\\
        3,& \text{if } s = 5
    \end{cases}
\end{align*}

\textbf{MTBench mapping:}
\begin{align*}
    f_{\text{MTBench}}(s) = 
    \begin{cases}
        1,& \text{if } s \in \{1,2\}\\
        2,& \text{if } s \in \{0,3\}\\
        3,& \text{if } s \in \{4,5\}\\
    \end{cases}
\end{align*}

The main difference between these mappings is the treatment of score 4. For TLDR, it is grouped with score 3 in the middle category, while for MTBench, it is grouped with score 5 in the top category.
These different groupings reflect the empirical distribution patterns observed in each dataset.
MTBench evaluations exhibit more separation between high-performing candidates, while TLDR shows finer distinctions between middle performing candidates.

\section{ADDITIONAL RESULTS}

\subsection{Prediction-Powered Inference Results}
\label{app:ppi-results}

Table~\ref{tab:ppi_results} presents ranking results using Prediction-Powered Inference (PPI) across GPQA, MMLU Pro, and Omni-MATH benchmarks with varying amounts of labeled data (5\% and 10\%).
PPI achieves high correlations with true rankings and strong coverage rates.
These results confirm that PPI is effective at quantifying uncertainty when calibrated with even small amounts of labeled data, providing reliable confidence intervals while maintaining competitive ranking accuracy.
The method's strength is its ability to use limited ground truth to correct for judge bias, making it well suited when some labeled examples are available.
In contrast, the Bayesian framework provides robust uncertainty quantification when gold-standard labels are unavailable, addressing the complementary scenario where obtaining any ground truth is impractical.

\begin{table*}
    \centering
    \caption{Performance of PPI ranking method across benchmark datasets with different labeled fractions.
    Results show Spearman correlation (Corr) and coverage rates (Cov) with 95\% confidence intervals.}
    \footnotesize
    \adjustbox{center}{%
    \begin{tabular}{c|cc|cc|cc}
    \toprule
    Labeled Fraction & \multicolumn{2}{c}{GPQA} & \multicolumn{2}{c}{MMLU Pro} & \multicolumn{2}{c}{Omni-MATH}\\
    & Corr & Cov & Corr & Cov & Corr & Cov \\
    \midrule
    5\% & \ci{0.773}{0.445}{0.889} & \ci{0.889}{0.778}{1.000} & \ci{0.902}{0.691}{0.919} & \ci{0.947}{0.749}{1.000} & \ci{0.789}{0.486}{0.905} & \ci{0.895}{0.789}{1.000} \\[0.5em]
    10\% & \ci{0.875}{0.529}{0.923} & \ci{0.944}{0.790}{1.000} & \ci{0.922}{0.782}{0.929} & \ci{1.000}{0.737}{0.947} & \ci{0.737}{0.621}{0.929} & \ci{0.895}{0.801}{1.000} \\[0.5em]
    \bottomrule
    \end{tabular}
    }
    \label{tab:ppi_results}
\end{table*}

\subsection{Posterior Distribution Plots}

Figures \ref{fig:mmlu_omni_combined}--\ref{fig:tldr_all_dimensions} present posterior distributions for candidate rankings across additional benchmarks, complementing the GPQA results in the main text. These visualizations reveal consistent patterns in how different methods quantify uncertainty.

\textbf{MMLU Pro and Omni-MATH} (Figure~\ref{fig:mmlu_omni_combined}): Both datasets exhibit similar characteristics to GPQA, with the Bayesian method producing wider credible intervals that capture epistemic uncertainty about judge quality.
The increased width is most pronounced for middle-ranked candidates, where judge disagreement tends to be highest. For Omni-MATH, greater overall uncertainty is observed across all methods, consistent with the sensitivity analysis showing this dataset's vulnerability to constancy violations.

\begin{figure}[h]
    \centering
    \includegraphics[width=0.45\linewidth]{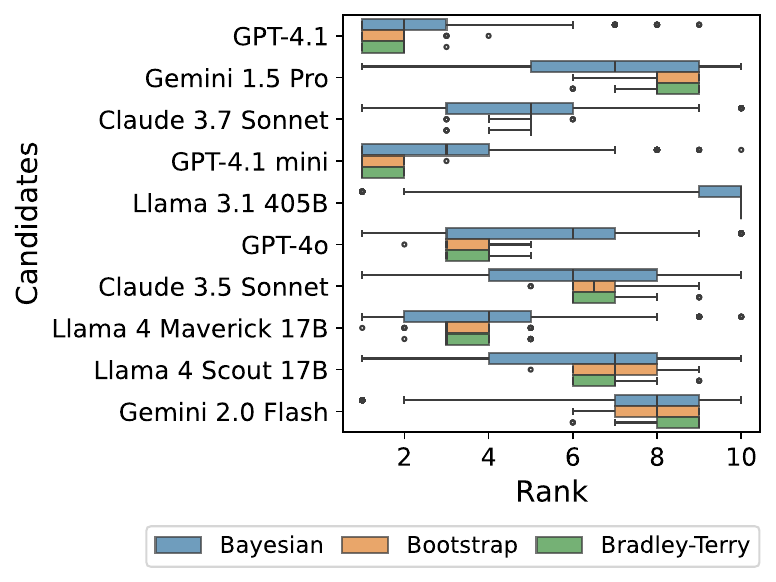}
    \hfill
    \includegraphics[width=0.45\linewidth]{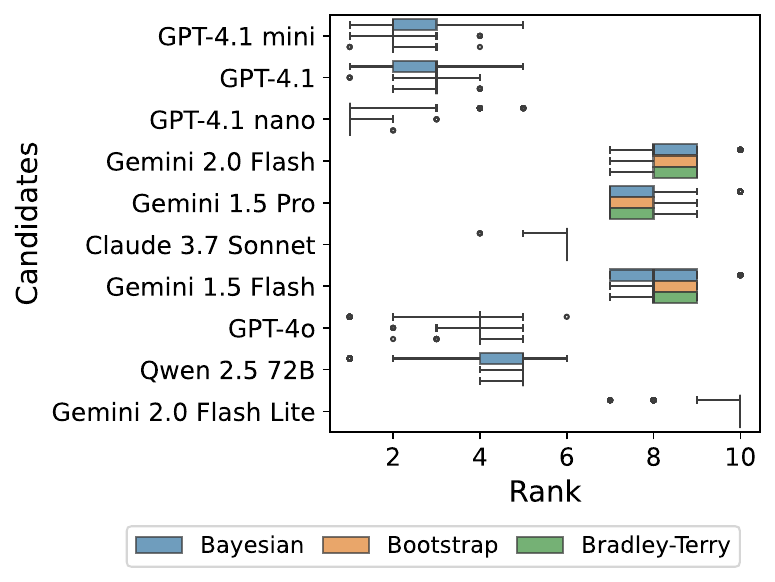}
    \caption{Posterior distributions for candidate rankings on MMLU Pro (left) and Omni-MATH (right), with candidates ordered by true ranking (best at top).}
    \label{fig:mmlu_omni_combined}
\end{figure}

\begin{figure}[h]
    \centering
    \includegraphics[width=0.5\linewidth]{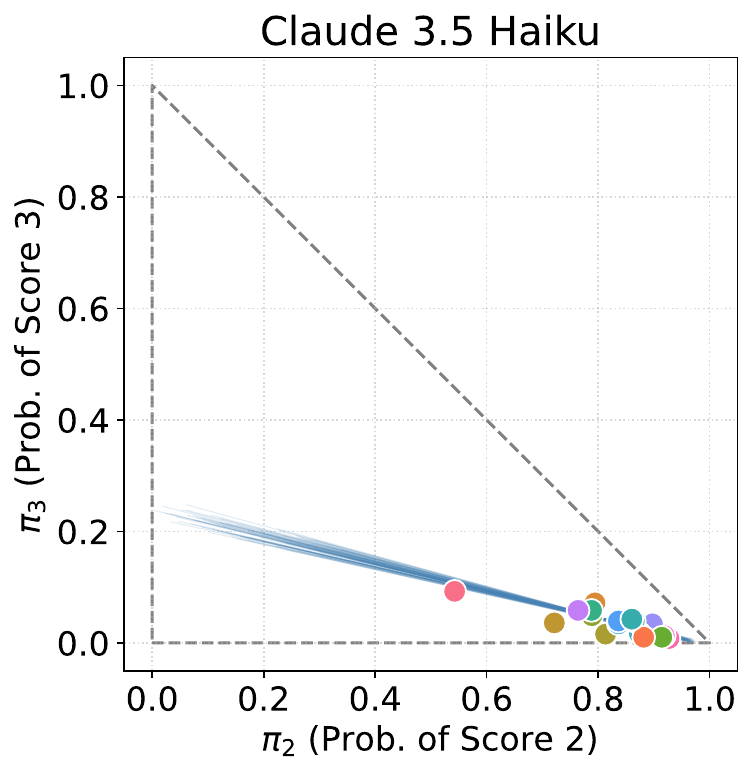}
    \caption{Candidates visualized on the probability simplex for MMLU Pro, with judge configurations sampled from the posterior (blue line segments). With binary true scores (correct/incorrect) and three assigned score categories (correct, incorrect, abstain), each posterior judge configuration maps to two points on the simplex connected by a line segment. The tight clustering of candidates in the upper region reflects the high overall accuracy on MMLU Pro, while the posterior judge configurations show moderate variability consistent with the sensitivity to constancy violations observed in the main text.}
    \label{fig:mmlu_pro_simplex}
\end{figure}

\textbf{SummEval} (Figure~\ref{fig:tldr_all_dimensions}): Across the four evaluation dimensions (coherence, consistency, fluency, and relevance), the Bayesian approach maintains wider intervals while preserving accurate point estimates.
The coherence and consistency dimensions show tighter clustering of posterior distributions compared to fluency and relevance, suggesting these aspects may be more reliably evaluated by LLM judges. 
The top-ranked candidates (M22, M23, M17) show relatively narrow intervals across all methods, indicating strong agreement on the best performers.

\begin{figure}
    \centering
    \includegraphics[width=0.5\linewidth]{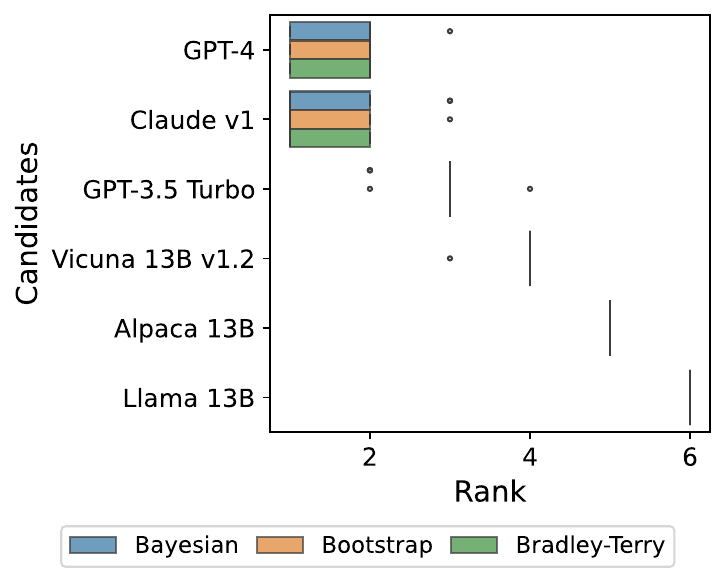}
    \caption{Posterior distributions for candidate rankings on MTBench, with candidates ordered by true ranking (best at top).}
    \label{fig:mtbench_boxplots}
\end{figure}

\textbf{MTBench} (Figure~\ref{fig:mtbench_boxplots}): The MTBench dataset exhibits tight posterior distributions across all methods, indicating strong judge consensus. 
Even the Bayesian approach produces relatively narrow credible intervals, suggesting low epistemic uncertainty for this dataset. 
This aligns with the sensitivity analysis showing MTBench maintains stable rankings under relaxed constancy assumptions, a case where judge-based evaluation is particularly reliable.

A key pattern across all datasets is that while the three methods generally agree on median rankings, they diverge substantially in uncertainty quantification.
The Bootstrap and Bradley-Terry methods produce similar narrow confidence intervals that account only for sampling variability. The Bayesian framework's wider intervals reflect both aleatoric and epistemic uncertainty, explaining its superior coverage rates reported in the main text. This difference is most pronounced for candidates with intermediate performance levels, where judge quality uncertainty has the greatest impact on ranking uncertainty.
\begin{figure}[ht]
    \centering
    \includegraphics[width=0.45\linewidth]{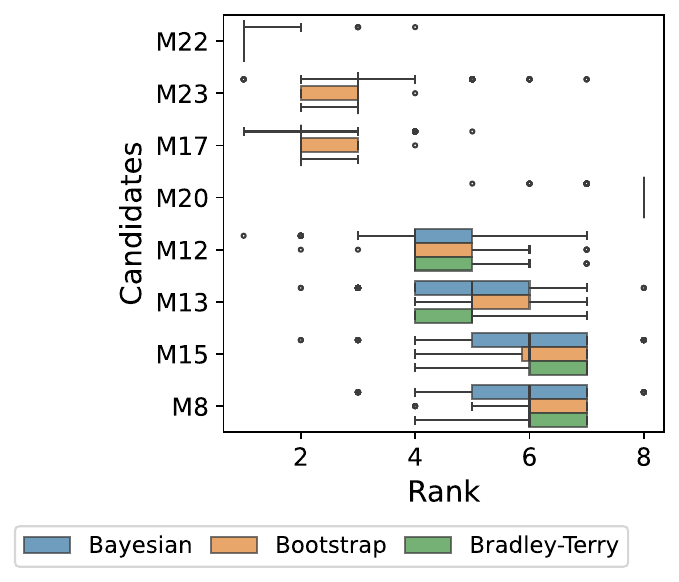}
    \hfill
    \includegraphics[width=0.45\linewidth]{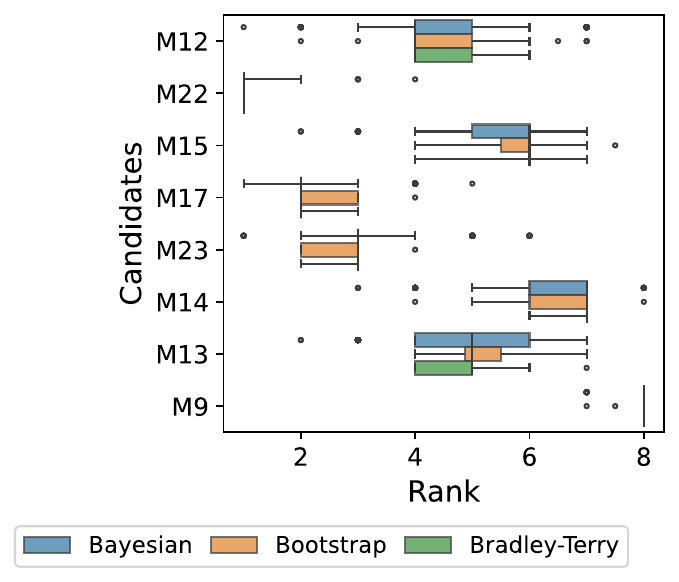}\\
    \vspace{0.2cm}
    \includegraphics[width=0.45\linewidth]{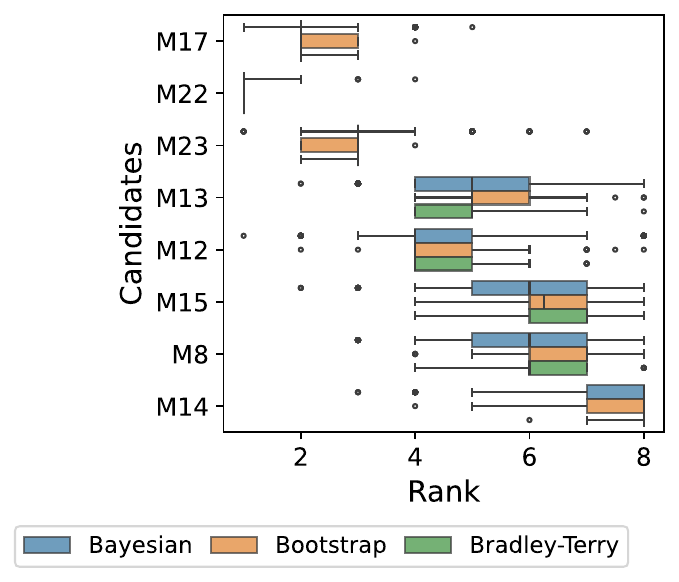}
    \hfill
    \includegraphics[width=0.45\linewidth]{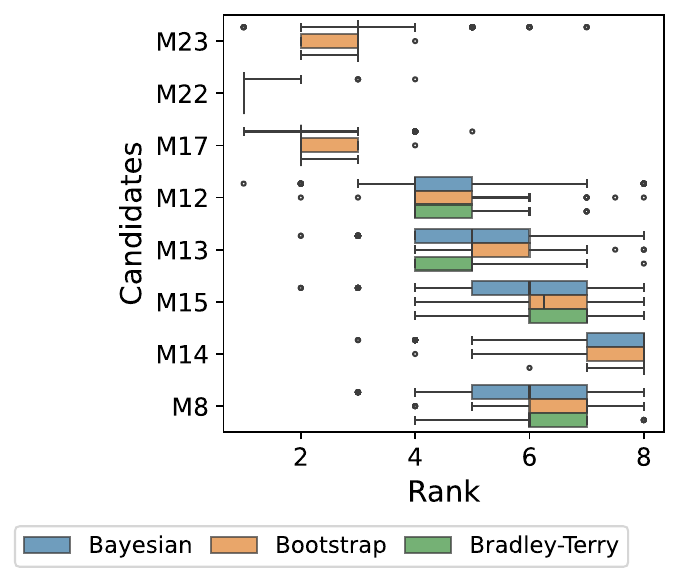}
    \caption{Posterior distributions for candidate rankings on TLDR/SummEval across four evaluation dimensions: coherence (top left), consistency (top right), fluency (bottom left), and relevance (bottom right). Candidates are ordered by true ranking (best at top).}
    \label{fig:tldr_all_dimensions}
\end{figure}

\subsection{Robustness to Self-Preference Bias}
\label{app:self-preference-simulation}

While the main experiments demonstrate the framework's ability to handle naturally occurring judge biases (e.g., MTBench self-preference in Table~\ref{tab:mtbench-self-pref}), additional controlled experiments were conducted to systematically evaluate robustness to self-preference bias.

\textbf{Experimental Design:} Using the GPQA dataset, self-preference bias was artificially induced by manipulating judge-assigned scores when judges evaluated candidates from the same model family.
Specifically, when Claude judges evaluated Claude candidates or GPT judges evaluated GPT candidates, the assigned score was shifted upward with probability 0.8. This creates a strong but realistic bias pattern similar to that observed in real evaluations.

\textbf{Bias Injection Protocol:}
\begin{itemize}
    \item For binary scores (correct/incorrect), incorrect answers were changed to correct with probability 0.8
    \item For abstentions, scores were converted to correct with probability 0.8
    \item Non-self evaluations remained unchanged
\end{itemize}

This manipulation creates a challenging scenario where traditional averaging methods would systematically overestimate the performance of judge-candidates.

\begin{table}[h]
\centering
\caption{Performance under simulated self-preference bias on GPQA dataset}
\label{tab:simulated-bias}
\begin{tabular}{lcc}
\toprule
Method & Correlation & Coverage \\
\midrule
Bayesian with self-preference adjustment & \textbf{0.811} & \textbf{0.852} \\
Bayesian without adjustment & 0.701 & 0.611 \\
Bootstrap & 0.680 & 0.519 \\
Bradley-Terry & 0.673 & 0.482 \\
Simple Average & 0.680 & 0.056 \\
Single Judge & 0.683 & 0.148 \\
\bottomrule
\end{tabular}
\end{table}

The results demonstrate that the Bayesian framework with self-preference adjustment effectively corrects for the induced bias, achieving correlation and coverage rates comparable to the unbiased setting shown in the main text. 
Without adjustment, all methods suffer substantial performance degradation, with the Bayesian method's correlation dropping from 0.881 to 0.701.

\subsection{Prior Specification Comparison}
\label{app:prior-comparison}

The weight-propagation prior (Section~\ref{sec:judge_quality}) encodes structural monotonicity through a directed graph over confusion matrix entries. It is compared against two alternatives: a diagonal-heavy Dirichlet prior that places independent priors on each confusion matrix row with boosted diagonal concentration, and a mixing prior that constructs each row as a convex combination of the preceding row and an innovation term. Table~\ref{tab:prior-comparison} reports Spearman rank correlation with ground truth and 95\% credible interval coverage for each specification across GPQA, MMLU~Pro, and Omni-MATH, with 95\% bootstrap confidence intervals over 50 resamples.

\begin{table}[h]
\centering
\caption{Spearman correlation and coverage under three prior specifications, with 95\% bootstrap confidence intervals. The weight-propagation prior is the only specification with consistently positive correlations and tight confidence intervals across all benchmarks.}
\label{tab:prior-comparison}
\footnotesize
\begin{tabular}{llcc}
\toprule
Benchmark & Prior & Spearman [95\% CI] & Coverage [95\% CI] \\
\midrule
GPQA & Weight-propagation & 0.920 [0.869, 0.946] & 0.833 [0.722, 0.889] \\
GPQA & Diagonal-heavy & $-0.558$ [$-0.932$, 0.920] & 0.944 [0.013, 0.944] \\
GPQA & Mixing & \textbf{0.921} [$-0.936$, 0.896] & 0.833 [0.056, 0.944] \\
\midrule
MMLU~Pro & Weight-propagation & \textbf{0.940} [0.840, 0.943] & 0.895 [0.895, 1.000] \\
MMLU~Pro & Diagonal-heavy & 0.809 [$-0.878$, 0.939] & 0.947 [0.749, 1.000] \\
MMLU~Pro & Mixing & $-0.886$ [$-0.906$, 0.312] & 0.842 [0.064, 0.947] \\
\midrule
Omni-MATH & Weight-propagation & \textbf{0.811} [0.730, 0.834] & 0.737 [0.421, 0.789] \\
Omni-MATH & Diagonal-heavy & 0.251 [$-0.769$, 0.783] & 0.579 [0.486, 0.895] \\
Omni-MATH & Mixing & $-0.758$ [$-0.817$, 0.801] & 0.737 [0.404, 0.895] \\
\bottomrule
\end{tabular}
\end{table}

Both alternative priors exhibit instability: their bootstrap confidence intervals span from strongly negative to strongly positive correlations, indicating that on any given data resample the ranking may be accurate or catastrophically inverted. Each fails through a distinct mechanism.

\paragraph{Confusion matrix collapse (diagonal prior).}
The diagonal-heavy prior places independent Dirichlet priors on each confusion matrix row with boosted diagonal elements but no inter-row coupling. On GPQA, where judge signal is weak and score marginals are near-balanced, both rows of the confusion matrix collapse toward the marginal distribution of observed scores. All candidates' latent prevalences compress into a narrow range, and the residual variation anti-correlates with ground truth ($\rho = -0.558$). The bootstrap CI [$-0.932$, $0.920$] spans nearly the full range, confirming this instability across resamples. On Omni-MATH, the same mechanism yields a point estimate of only $0.251$ with a CI of [$-0.769$, $0.783$].

\paragraph{Label switching (mixing prior).}
The mixing prior constructs each confusion matrix row as a convex combination of the previous row and an innovation term but lacks monotonicity constraints. On MMLU~Pro, the posterior finds a label-switched solution where $P(\text{correct} \mid \text{true}=\text{low}) = 0.839$ exceeds $P(\text{correct} \mid \text{true}=\text{high}) = 0.633$, inverting the labels entirely ($\rho = -0.886$). This is not a convergence failure (only 15 of 4000 transitions are divergent) but a likelihood geometry problem: the label-switched mode is equally valid without structural constraints. The same failure occurs on Omni-MATH ($\rho = -0.758$).

\paragraph{Why weight-propagation succeeds.}
The weight-propagation prior makes label switching geometrically impossible: probability mass can only shift toward higher assigned scores as the true score increases, so the confusion matrix rows are ordered by construction. This structural monotonicity also prevents collapse by maintaining meaningful separation between rows ($L_1$ separation $> 1.3$ vs.\ $0.25$--$0.41$ for the failing priors). Critically, the weight-propagation prior is the only specification whose bootstrap confidence intervals remain consistently positive across all three benchmarks (Table~\ref{tab:prior-comparison}), confirming its stability under data perturbation.

\section{PROOFS}

This section presents the proofs for the main results.
For the identifiability results, recall the following assumptions:
\begin{assumption}
Judge $\hat{s}_j$ satisfies ``strong constancy'' if its confusion matrix is the same across all $K$ candidates: For each $m$, there is some $\theta^{(j)}_m$ such that
$\theta^{(j)}_{m,k} = \theta^{(j)}_m$ for $k=1,\cdots,K$.
\
\vspace{-0.1cm}
\end{assumption}
\begin{assumption}
Judge $\hat{s}_j$ satisfies ``moderate constancy'' if its confusion matrix is the same for all non-self candidates: For each $m$, there is some $\theta^{(j)}_m$ such that $\theta^{(j)}_{m,k} = \theta^{(j)}_m$ for all $k\ne j$.
\
\end{assumption}
\begin{assumption}
The $j$-th judge's probability of assigning the lowest score when the true score is equal to $m$ decreases with respect to $m$.
\label{assumption:no_stupid1}
\end{assumption}

\subsection{Proof for Theorem~\ref{thrm:strong_constancy2}}
\begin{proof}[Proof for Theorem~\ref{thrm:strong_constancy2}(i)]
By Assumption~\ref{assumption:constancy}, we have
\[
{{\theta_1}} \;=\;\Pr\bigl(\hat S_k=1\mid S_k^*=1\bigr),
\qquad
{{\theta_0}} \;=\;\Pr\bigl(\hat S_k=1\mid S_k^*=0\bigr)
\]
for all $k=1,\cdots, K$.
Moreover, by Assumption~\ref{assumption:no_stupid1}, we have that $\theta_1 > \theta_0$.

For candidate $k$, let 
\[
\pi_k \;=\;\Pr\bigl(S_k^*=1\bigr),
\qquad
\gamma_k \;=\;\Pr\bigl(\hat S_k=1\bigr).
\]
Then
\[
\gamma_k
=\;
\pi_k\,{{\theta_1}} \;+\;(1-\pi_k)\,{{\theta_0}}.
\]

For any set of $K$ candidates with non-equal prevalences of the judge-assigned scores, suppose WLOG that the ordering indices $k_1,\cdots, k_K$ are such that $\gamma_{k_1} > \gamma_{k_2} > \cdots > \gamma_{k_K}$.
This implies that
\[
\bigl[{{\theta_0}} + \pi_{k_1}({{\theta_1}}-{{\theta_0}})\bigr]
\;>\;
\bigl[{{\theta_0}} + \pi_{k_2}({{\theta_1}}-{{\theta_0}})\bigr]
\;>\;
\cdots
\;>\;
\bigl[{{\theta_0}} + \pi_{k_K}({{\theta_1}}-{{\theta_0}})\bigr]
\]
which implies that
\[
\pi_{k_1} > \pi_{k_2} > \cdots > \pi_{k_K}.
\]
\end{proof}

\begin{proof}[Proof for Theorem~\ref{thrm:strong_constancy2}(ii)]
WLOG, let the candidates who are not also judges (referred to as \emph{core candidates}) have indices $k = 3,4,\cdots, J$.
Let the two LLMs who are both judges and candidates be $k=j\in \{1, 2\}$ (referred to as \textit{judge-candidates}).

Per Assumption~\ref{assumption:nonself_constancy}, we have for each judge $j = 1, 2 $ that there exists $\theta_1^{(j)} > \theta_0^{(j)}$ such that
\[
{{\theta_1^{(j)}}} \;=\;\Pr\bigl(\hat S_k=1\mid S_k^*=1\bigr),
\qquad
{{\theta_0^{(j)}}} \;=\;\Pr\bigl(\hat S_k=1\mid S_k^*=0\bigr),
\]
for all core candidates $k$.

Applying our result for Theorem~\ref{thrm:strong_constancy2}(i), we can rank core candidates $k=3$ and $k=4$ using their ranking on the line segment from $\theta_1^{(1)}$ to $\theta_0^{(1)}$ (or the line segment from $\theta_1^{(2)}$ to $\theta_0^{(2)}$). 
Then all core-candidates $k$ can then be jointly ranked (including $k=3,4$) by assessing:
\[
\frac{\gamma^{(1)}_k - \gamma^{(1)}_k}{\gamma^{(1)}_4 - \gamma^{(1)}_3}
=
\frac{
\pi_{k} - \pi_3
}{
\pi_{4} - \pi_3
}.
\]
We can rank judge-candidates in a similar way, using their non-self-judged position.
Specifically, for candidate $k=1$, we use its score distribution from judge $j=2$, i.e. $\gamma^{(2)}_1 = \Pr(\hat{S}_1^{(2)} = 1)$, relative to candidates $k=3$ and $k=4$ to compute
\[
\frac{\gamma^{(2)}_1 - \gamma^{(2)}_3}{\gamma^{(2)}_4 - \gamma^{(2)}_3}
=
\frac{
\pi_{1} - \pi_3
}{
\pi_{4} - \pi_3
}.
\]
Likewise, for candidate $k=2$, we use its score distribution from judge $j=1$, i.e. $\gamma^{(1)}_2 = \Pr(\hat{S}_2^{(1)} = 1)$, relative to candidates $k=3$ and $k=4$ to compute
\[
\frac{\gamma^{(1)}_2 - \gamma^{(1)}_3}{\gamma^{(1)}_4 - \gamma^{(1)}_3}
=
\frac{
\pi_{2} - \pi_3
}{
\pi_{4} - \pi_3
}.
\]
Thus we have
\[
\frac{\gamma^{(1)}_2 - \gamma^{(1)}_3}{\gamma^{(1)}_4 - \gamma^{(1)}_3}
=
\frac{
\pi_{1} - \pi_3
}{
\pi_{2} - \pi_3
}.
\]
By ranking these shifted and scaled judge-assigned score distributions, we can recover the true ranking between all candidates.
\end{proof}

\subsection{Necessary and Sufficient Conditions for 2-Level Systems}
\label{app:necessary-sufficient}

While Theorem~\ref{thrm:strong_constancy2} provides sufficient conditions for ranking identifiability under constancy assumptions, the exact necessary and sufficient conditions for when rankings are preserved in 2-level systems with a single judge can also be characterized.

\begin{lemma}
\label{lem:necessary-sufficient}
Let $(k_1, k_2, \ldots, k_K)$ be a permutation such that the true scores satisfy
$\Pr(S_{k_1}^*=2) > \Pr(S_{k_2}^*=2) > \cdots > \Pr(S_{k_K}^*=2)$.
Then the following are equivalent:
\begin{enumerate}
\item[(i)] The judge-assigned scores preserve this ordering: $\Pr(\hat{S}_{k_1}^{(j)}=2) > \Pr(\hat{S}_{k_2}^{(j)}=2) > \cdots > \Pr(\hat{S}_{k_K}^{(j)}=2)$
\item[(ii)] For all $\ell \in \{1, \ldots, K-1\}$: $\Pr(\hat{S}_{k_\ell}^{(j)}=2) > \Pr(\hat{S}_{k_{\ell+1}}^{(j)}=2)$
\end{enumerate}
\end{lemma}

\begin{proof}
The forward direction (i $\Rightarrow$ ii) is immediate.
For (ii $\Rightarrow$ i), assume $Pr(\hat{S}_{k_\ell}=2)>Pr(\hat{S}_{k_{\ell+1}}=2)$ for every $\ell=1,\dots,K-1$.
Fix any indices $1\le i<j\le K$. By repeated transitivity,
\[
\hat p_{k_i} > \hat p_{k_{i+1}} > \cdots > \hat p_{k_j},
\]
hence $\hat p_{k_i} > \hat p_{k_j}$. Since this holds for all $i<j$, we obtain the strict chain
$\hat p_{k_1} > \hat p_{k_2} > \cdots > \hat p_{k_K}$, proving (i).
\end{proof}

The key insight is that only $K-1$ local inequalities between adjacent candidates need to be verified, rather than all $\binom{K}{2}$ pairwise comparisons.
For condition (ii) to hold given the true ranking, the following is required:

\begin{theorem}[Necessary and Sufficient Conditions]
\label{thrm:necessary-sufficient}
Consider a 2-level scoring system with a single judge.
True rankings are preserved by judge $j$ for all possible candidate configurations if and only if for each adjacent pair $(k_\ell, k_{\ell+1})$ in the true ranking:
\begin{align}
&\theta_{2,k_\ell}^{(1)} > \theta_{1,k_{\ell+1}}^{(1)} \quad \text{(cross-term dominance)} \label{eq:cross-term}\\
&\theta_{1,k_\ell}^{(1)} \geq \theta_{1,k_{\ell+1}}^{(1)}, \text{ and if equal, then } \theta_{2,k_\ell}^{(1)} > \theta_{1,k_\ell}^{(1)} \label{eq:fpr-condition}\\
&\theta_{2,k_\ell}^{(1)} \geq \theta_{2,k_{\ell+1}}^{(1)}, \text{ and if equal, then } \theta_{2,k_{\ell+1}}^{(1)} > \theta_{1,k_{\ell+1}}^{(1)} \label{eq:tpr-condition}
\end{align}
where $\theta_{m,k}^{(1)} = \Pr(\hat{S}_k^{(1)}=2|S_k^*=m)$ represents the judge's confusion matrix entries.
\end{theorem}

\begin{proof}
Condition (ii) from Lemma~\ref{lem:necessary-sufficient} states that $\Pr(S_{k_\ell}^*=2) > \Pr(S_{k_{\ell+1}}^*=2)$ implies $\Pr(\hat{S}_{k_\ell}^{(1)}=2) > \Pr(\hat{S}_{k_{\ell+1}}^{(1)}=2)$.
Using the notation $\pi_{k,m} = \Pr(S_k^*=m)$ and $\theta_{m,k}^{(1)} = \Pr(\hat{S}_k^{(1)}=2|S_k^*=m)$, this becomes:
\begin{align}
\pi_{k_\ell,2} > \pi_{k_{\ell+1},2} \implies \gamma_{k_\ell}^{(1)} > \gamma_{k_{\ell+1}}^{(1)}
\label{eq:monotone-condition}
\end{align}
where $\gamma_k^{(1)} = \theta_{1,k}^{(1)}(1-\pi_{k,2}) + \theta_{2,k}^{(1)}\pi_{k,2}$ is the judge-assigned score probability.

For \eqref{eq:monotone-condition} to hold, we need the difference
\begin{align}
g(\pi_{k_\ell,2}, \pi_{k_{\ell+1},2}) &= \gamma_{k_\ell}^{(1)} - \gamma_{k_{\ell+1}}^{(1)}\nonumber\\
&= \theta_{1,k_\ell}^{(1)}(1-\pi_{k_\ell,2}) + \theta_{2,k_\ell}^{(1)}\pi_{k_\ell,2} - \theta_{1,k_{\ell+1}}^{(1)}(1-\pi_{k_{\ell+1},2}) - \theta_{2,k_{\ell+1}}^{(1)}\pi_{k_{\ell+1},2}\nonumber\\
&= (\theta_{1,k_\ell}^{(1)}-\theta_{1,k_{\ell+1}}^{(1)}) + (\theta_{2,k_\ell}^{(1)}-\theta_{1,k_\ell}^{(1)})\pi_{k_\ell,2} + (\theta_{1,k_{\ell+1}}^{(1)}-\theta_{2,k_{\ell+1}}^{(1)})\pi_{k_{\ell+1},2}
\label{eq:g-function}
\end{align}
to be positive for all $\pi_{k_\ell,2} > \pi_{k_{\ell+1},2}$ with $0 \leq \pi_{k_{\ell+1},2} \leq \pi_{k_\ell,2} \leq 1$.

Since $g(\pi_{k_\ell,2}, \pi_{k_{\ell+1},2})$ is linear in both arguments, it is positive throughout the feasible region if and only if it is non-negative at all extreme points. The feasible region $\{(\pi_{k_\ell,2}, \pi_{k_{\ell+1},2}): 0 \leq \pi_{k_{\ell+1},2} \leq \pi_{k_\ell,2} \leq 1\}$ has three extreme points:

\textbf{At $(0,0)$:} $g(0,0) = \theta_{1,k_\ell}^{(1)}-\theta_{1,k_{\ell+1}}^{(1)} \geq 0$

\textbf{At $(1,0)$:} $g(1,0) = \theta_{1,k_\ell}^{(1)}-\theta_{1,k_{\ell+1}}^{(1)} + (\theta_{2,k_\ell}^{(1)}-\theta_{1,k_\ell}^{(1)}) = \theta_{2,k_\ell}^{(1)} - \theta_{1,k_{\ell+1}}^{(1)} > 0$

\textbf{At $(1,1)$:} $g(1,1) = \theta_{1,k_\ell}^{(1)}-\theta_{1,k_{\ell+1}}^{(1)} + (\theta_{2,k_\ell}^{(1)}-\theta_{1,k_\ell}^{(1)}) + (\theta_{1,k_{\ell+1}}^{(1)}-\theta_{2,k_{\ell+1}}^{(1)}) = \theta_{2,k_\ell}^{(1)} - \theta_{2,k_{\ell+1}}^{(1)} \geq 0$

Note that $g(1,0)$ must be strictly positive since this corresponds to the interior of the feasible region where $\pi_{k_\ell,2} > \pi_{k_{\ell+1},2}$.

For strict inequality in \eqref{eq:monotone-condition} when $\pi_{k_\ell,2} > \pi_{k_{\ell+1},2}$, we also need $g$ to be strictly positive when moving slightly away from the boundary points. This requires:
\begin{itemize}
    \item If $\theta_{1,k_\ell}^{(1)} = \theta_{1,k_{\ell+1}}^{(1)}$, then $\theta_{2,k_\ell}^{(1)} > \theta_{1,k_\ell}^{(1)}$ to ensure $g(\Delta, 0) > 0$ for small $\Delta > 0$
    \item If $\theta_{2,k_\ell}^{(1)} = \theta_{2,k_{\ell+1}}^{(1)}$, then $\theta_{2,k_{\ell+1}}^{(1)} > \theta_{1,k_{\ell+1}}^{(1)}$ to ensure $g(1, 1-\Delta) > 0$ for small $\Delta > 0$
\end{itemize}

These conditions correspond exactly to equations \eqref{eq:cross-term}, \eqref{eq:fpr-condition}, and \eqref{eq:tpr-condition}.
\end{proof}

\begin{remark}
\label{rem:constancy-interpretability}
These necessary and sufficient conditions reveal that without constancy assumptions, ranking preservation requires complex coordination between confusion matrices: higher-ranked candidates must have both larger false-positive rates (FPR) and true-positive rates (TPR) than lower-ranked candidates, with the higher rank's TPR exceeding the lower rank's FPR.
\end{remark}

\subsection{Proof for Theorem~\ref{thrm:strong_constancy3}}

The proof is given for three levels; the extension to more levels is straightforward.

\begin{proof}[Proof for Theorem~\ref{thrm:strong_constancy3}]

The proof is first given for a single judge.
It suffices to show that there exist two candidates whose relative rankings cannot be identified from the judge-assigned score distribution alone.

Consider two candidates in the strict interior of the probability simplex, whose marginal judge-assigned score distributions are denoted $\gamma_1$ and $\gamma_2$.
Let us first consider some judge whose vertices are denoted $\Theta = \left (
\begin{matrix}
\theta_{1} & \theta_{2}&  \theta_{3}
\end{matrix} \right)$, whose vertices are also strictly in the interior of the probability simplex and where $\Theta$ has full column rank.
Then the distribution of the true scores (the barycentric coordinates) for candidate $k$ is given by $\pi_k = \Theta^{-1} \gamma_k$.
The true scores of the two candidates are thus equal if
\begin{align}
\left(\Theta^{-1} (\gamma_1 - \gamma_2)\right)^\top
\left(
\begin{matrix}
    0\\1\\2
\end{matrix}
\right)
= 0.
\label{eq:equal_score}
\end{align}

Now it is possible that $\gamma_1$ and $\gamma_2$ were generated by a slightly different judge with triangle corners defined by $\Theta'_h = \Theta + h\Delta$, where $h \in \mathbb{R}$ and $\Delta$ is any matrix such that the columns sum to zero.
(As long as $h$ is sufficiently small, $\Theta'_h$ is a valid set of judge vertices.)
If this other judge were the true judge, then the difference of the true scores between the two candidates would be given by
\begin{align}
    \left(\left(\Theta + h\Delta\right)^{-1} (\gamma_1 - \gamma_2)\right)^\top
    \left(
\begin{matrix}
    0\\1\\2
\end{matrix}
\right).
\label{eq:score_diff}
\end{align}

To prove nonidentifiability of the rankings, it thus suffices to show that for any $\Theta$, there exists $\gamma_1, \gamma_2, \Delta$ such that the score difference \eqref{eq:score_diff} is zero at $h=0$, i.e. \eqref{eq:equal_score} holds, and the derivative of the score difference \eqref{eq:score_diff} is nonzero, i.e.
\begin{align}
\nabla_{h}
    \left(\left(\Theta + h\Delta\right)^{-1} (\gamma_1 - \gamma_2)\right)^\top
    \left(
\begin{matrix}
    0\\1\\2
\end{matrix}
\right)
\propto
\left(\Theta^{-1}(\gamma_1 - \gamma_2)\right)^\top
(\Theta^{-1}\Delta)^{\top}
\left(
\begin{matrix}
    0\\1\\2
\end{matrix}
\right)
\ne 0.
\label{eq:score_der}
\end{align}
If this were to hold, then there exists some $h > 0$ such that the true score rankings between two candidates with marginal score distributions $\gamma_1$ and $\gamma_2$ from a judge with vertices $\Theta - h \Delta$ would be the opposite if the judge instead had vertices $\Theta + h \Delta$.


To find such a $\gamma_1, \gamma_2$, and $\Delta$, let $\bar{\pi} = (\frac{1}{3},\frac{1}{3},\frac{1}{3})^T$ and $\pi_1 = \bar{\pi} + \frac{\epsilon a}{2}$, $\pi_2 = \bar{\pi} -\frac{\epsilon a}{2}$, where $a = (1,-2,1)^T$, and $\epsilon >0$ and small enough such that $\pi_1$ and $\pi_2$ belong to the probability simplex. Note that $\pi_1^T\mathbf{1} = \pi_2^T\mathbf{1} = \bar{\pi}^T\mathbf{1} = 1$ since $a^T\mathbf{1}=0$.

Let $\gamma_k = \Theta \pi_k$ for $k= 1,2$.
Then $\eqref{eq:score_diff} = 0$ at $h=0$. 
To prove that $\eqref{eq:score_der} \ne 0$ at $h=0$, let $u = (\Theta^{-1})^T\left(
\begin{matrix}
    0\\1\\2
\end{matrix}
\right)$ and $\bar{u} = \frac{\mathbf{1}^Tu}{3}$. Define
\begin{align*}
    \Delta = (u - \bar{u}\mathbf{1})a^\top \in \mathbb{R}^{3\times 3}.
\end{align*}
Because
$
    \mathbf{1}^\top (u - \bar{u}\mathbf{1})a^T = \vec{0}^\top,
$
we have constructed a $\Delta$ such that every column sums to 0.
Moreover, because $\Theta^{-1}(\gamma_1 - \gamma_2) = \epsilon\alpha$,  $\eqref{eq:score_der}$ simplifies to
\begin{align*}
   \left(\Theta^{-1}(\gamma_1 - \gamma_2)\right)^\top
(\Theta^{-1}\Delta)^{\top}
\left(
\begin{matrix}
    0\\1\\2
\end{matrix}
\right)
&=  u^T\Delta\epsilon a = u^T(u - \bar{u}\mathbf{1})a^Ta \epsilon\\
& = \sum_{m=1}^3(u_m - \bar{u})^2 \lVert a\rVert^2\epsilon \\
& \ge 0.
\end{align*}

Now we prove that the inequality is in fact strict.
We do this by contradiction.
In particular, note that the equity holds if and only if $ \vec{u} = (\Theta^{-1})^T \left(
\begin{matrix}
    0\\1\\2
\end{matrix}
\right) = c\vec{1}$ for some $c\in \mathbb{R}$. 
Left-multiplying all elements in the equality by $\Theta^\top$, this is equivalent to assuming
\begin{align*}
\Theta^\top \vec{u} =  \left(
\begin{matrix}
    0\\1\\2
\end{matrix}
\right) = c \Theta^\top \vec{1}.
\end{align*}
However, because each column of $\Theta$ is a probability vector, this would imply that $(0,1,2) = c \vec{1}$, which would be a contradiction.
This implies the score derivative is strictly nonzero.

Note that by adding more judges, the non-identifiability results remain the same. We can still find candidates and judges that satisfy the above conditions.

\end{proof}

\end{document}